\documentclass{article}

\usepackage{microtype}
\usepackage{booktabs} % for professional tables

\usepackage{hyperref}

\usepackage[accepted]{icml2022}

\usepackage{graphicx}
\usepackage{caption, subcaption}
\usepackage{tikz}
\usetikzlibrary{positioning}
\usetikzlibrary{bayesnet}

\usepackage{amsmath}
\usepackage{amssymb}
\usepackage{amsthm}
\usepackage{bm}
\usepackage{mathtools}

\usepackage[capitalize,noabbrev]{cleveref}

\theoremstyle{definition}

\theoremstyle{plain}
\newtheorem{theorem}{Theorem}
\newtheorem{lemma}{Lemma}
\newtheorem{corollary}{Corollary}

\newtheorem{claim}{Claim}
\newtheorem*{claim*}{Claim}
\newtheorem*{theorem*}{Theorem}

\usepackage[textsize=tiny]{todonotes}

\icmltitlerunning{On the Convergence of the Shapley Value in Parametric Bayesian Learning Games}

\newcommand{\V}{\bm{\mathcal{V}}}
\newcommand{\KL}[2]{\textnormal{KL}(#1\|#2)}
\newcommand{\KLE}[2]{\textnormal{KL}^+(#1\|#2)}
\newcommand{\TV}[2]{\bigl\| #1 - #2 \bigr\|_{\textnormal{TV}}}
\newcommand{\goinp}{\overset{\textnormal{p}}{\longrightarrow}}

\begin{document}

\twocolumn[
\icmltitle{On the Convergence of the Shapley Value in\\Parametric Bayesian Learning Games}




\begin{icmlauthorlist}
\icmlauthor{Lucas Agussurja}{soc}
\icmlauthor{Xinyi Xu}{soc,astar}
\icmlauthor{Bryan Kian Hsiang Low}{soc}%\\
%\textrm{lagussurja@gmail.com}, \{xuxinyi, lowkh\}@comp.nus.edu.sg
%sorry, it is not common to put the emails, so i had it removed.
\end{icmlauthorlist}
\icmlaffiliation{soc}{Department of Computer Science, National University of Singapore, Singapore.}
\icmlaffiliation{astar}{Institute for Infocomm Research, A$^*$STAR, Singapore}

\icmlcorrespondingauthor{Bryan Kian Hsiang Low}{lowkh@comp.nus.edu.sg}

\icmlkeywords{Bayesian learning, Shapley value, Finite sample convergence}

\vskip 0.3in
]


\printAffiliationsAndNotice{}  % leave blank if no need to mention equal contribution

\begin{abstract}
Measuring contributions is a classical problem in cooperative game theory where the Shapley value is the most well-known solution concept. In this paper, we establish the convergence property of the Shapley value in \textit{parametric Bayesian learning games} where players perform a Bayesian inference using their combined data, and the posterior-prior KL divergence is used as the characteristic function. We show that for any two players, under some regularity conditions, their difference in Shapley value converges in probability to the difference in Shapley value of a limiting game whose characteristic function is proportional to the log-determinant of the joint Fisher information. As an application, we present an online collaborative learning framework that is asymptotically Shapley-fair. Our result enables this to be achieved without any costly computations of posterior-prior KL divergences. Only a consistent estimator of the Fisher information is needed. The effectiveness of our framework is demonstrated with experiments using real-world data. 

% In collaborative Bayesian learning, however, it is unclear what a suitable characteristic (data valuation) function is. 
% Furthermore, there is the question of how Fisher information, which has been long recognized as a measure of information content in a dataset, can be utilized. 
% In this paper, we shed some light on these questions. In particular, we define parametric Bayesian learning games where players perform joint Bayesian inference using their combined data and a common prior. The characteristic function is then given by the posterior-prior KL divergence. We show that for any two players, under some regularity conditions, their difference in Shapley value converges in probability to the difference in Shapley value of a limiting game whose characteristic function is proportional to the log-determinant of the joint Fisher information. As an application, we present an online collaborative learning framework that is asymptotically Shapley-fair. Although the fairness is established in terms of posterior-prior KL divergences, it does not require any Bayesian inference, only a consistent estimator of the Fisher information. Its effectiveness is demonstrated with experiments using real-world data.
\end{abstract}

\section{Introduction}

Recent significant increase in computing power has enabled the training of high-capacity machine learning models (e.g.,~deep neural networks) which results in an unprecedented level of performances in many domains. A key ingredient of this success is the huge quantity of data that is used to train these models. Going forward, as models and tasks grow in complexity, the demand for more data will only rise in intensity and it will become necessary to utilize data from multiple sources and across different organizations. A benefit of data sharing is the ability to achieve a higher convergence rate. It has been shown that when collaboration is possible, an exponential improvement can be achieved in terms of sample complexity \cite{blum2017, chen2018, nguyen2018}.

When learning is done using datasets from multiple sources, a natural question arises: \emph{What are the relative contributions of individual datasets to the resulting model?} As examples, consider the following two scenarios: (a) In a hospital, a new study is to be conducted on the general population and datasets are collected from multiple clinics within the hospital. The datasets contain bias due to the specializations of the clinics.
Datasets from pediatricians, for example, will contain data for patients under $12$ years old, while datasets from rheumatologists will be gender-biased since women are (biologically) more likely to develop rheumatoid-type conditions. It is therefore important to know the influences that these datasets have on the outcome. (b) A group of companies have decided to jointly develop a machine learning application by pooling their data together. The application will then be provided as a service to customers (e.g., on a pay-per-query basis). After deployment, the group is then faced with the problem of distributing the generated revenue from the application. Though resolving these problems involves a number of non-technical (e.g., social and business) considerations, a quantitative measure of the datasets' contributions to the learning outcome would be crucial.

The problem of attribution (cost/profit sharing) is a classical problem in cooperative game theory and the two most well-known solution concepts are the Shapley \cite{shapley1953} and Banzhaf \cite{banzhaf1965} values. Both concepts are based on the marginal contribution, i.e., the increase in a coalition's value when a player joins it. These values of coalitions are given by the characteristic (valuation) function which maps each given coalition to a real number. Both the Shapley and Banzhaf values quantify a player's contribution with a weighted sum of its marginal contributions to all possible coalitions, but they differ in the weights used.
What makes the Shapley value appealing is its uniqueness in satisfying the desirable properties of symmetry, null player, linearity, and efficiency simultaneously. The first two properties are often used as criteria for fairness. The Banzhaf value satisfies these properties except efficiency. 

The use of Shapley value has gained traction in the machine learning community \cite{sim2020,rachxinyi,tay2021,zhaoxuan22,xu2021gradient,xu2021vol}. Applications that are related to the present work include quantifying the importance of features \cite{covert2020,lundberg2017}  and sets of training data \cite{ghorbani2019}. In these works, the characteristic function is defined as the performance of the trained models on some test set. However, the use of a test set can be infeasible. Consider the cases when the players' data, albeit can be explained by the same underlying hypothesis, may not come from the same sample space or may be generative in nature. In these cases, a Bayesian approach is a better fit. Instead of agreeing on a common test set, the players agree on a common prior. Joint Bayesian inferences are then performed and the value of a coalition is the information gain of their joint posterior from the prior. To our knowledge, the use and analysis of the Shapley value in this setting is still unexplored. So, the first contribution of our work here is to bridge this gap.

Specifically, we study the asymptotic properties of the Shapley/Banzhaf value in parametric Bayesian learning games. In these games, players come together to perform  Bayesian inference using their combined datasets. We assume that there is a true hypothesis (parameter) and a player's dataset is generated from a distribution that is indexed by the true parameter. In this work, the players' distributions need not be the same. Given a commonly agreed prior, the characteristic function is defined as the posterior-prior KL divergence. In our main result, we show that as the size of datasets tends to infinity, the difference in Shapley value converges in probability to the difference in Shapley value of a limiting game whose characteristic function is proportional to the log-determinant of the joint Fisher information. Fisher information is classically interpreted as a measure of information that a dataset carries about the underlying hypothesis. However, supposing one wants to use it as the characteristic function, it is unclear what the suitable statistic is. Providing this link is the second contribution of this work.

As the third contribution, we present a $2$-player online collaborative learning framework that is asymptotically Shapley-fair. The key idea underlying our framework is to vary the rates at which players' data are received such that the determinants of their Fisher information are the same in the limit. Applying the main result, we show that the difference in Shapley value converges in probability to zero. We empirically verify this on real-world data for Bayesian linear regression and mean estimation in a learned latent space. Moreover, we extend our framework to multiple players and empirically show that the convergence can also be achieved approximately with more than $2$ players.

\section{Parametric Bayesian Learning Game} \label{sec:bayesian-learning-game}

In this section, we will define a parametric Bayesian learning game and state the assumptions \textbf{A1} to \textbf{A4} on the players' data generation model.
We will first describe the players' data generation model. In this work, we assume that there is a true parameter (hypothesis) $\theta^*$ that all the players are interested in finding. In the parametric setting, $\theta^*$ lies in the parameter space $\Theta$ which is a subset of the $k$-dimensional Euclidean space. Each player $i$ is associated with a family of distributions $\{F_{i,\theta}\}_{\theta \in \Theta}$ and  observes the random variables
\[
\textbf{X}_i^m := (\textbf{X}_{i1}, \dots, \textbf{X}_{im}) \quad\text{where}\quad \textbf{X}_{i1}, \dots, \textbf{X}_{im} \overset{\text{i.i.d.}}{\sim} F_{i,\theta^*}
\]
and $m$ denotes the number of data points. We will now state the first three assumptions with regard to the players' data generation model.

(\textbf{A1}) The density of $F_{i,\theta}$ is twice differentiable in $\theta$.

(\textbf{A2}) Besides the independence of a player's data points, the random variables $\textbf{Y}_{i,\theta}\sim F_{i,\theta}$ of player $i$ and $\textbf{Y}_{j,\theta}\sim F_{j,\theta}$ of player $j$ are conditionally independent given $\theta\in\Theta$.

The third assumption states that the data generated from the true parameter can be distinguished from the ones generated from any other parameter. This is a common assumption in statistical analysis to guarantee the existence of a consistent estimator (e.g.,~maximum likelihood estimator). If the assumption is violated, then there exists another parameter $\theta\in\Theta$ that is observationally equivalent to $\theta^*$. Consequently, no algorithm will be able to distinguish between the two, even with an infinite quantity of data. Sometimes, the stronger assumption of identifiability is given, i.e., $F_{i,\theta} \neq F_{i,\theta'}$ (w.r.t.~a probability metric) whenever $\theta\neq\theta'$. Identifiability implies the following weaker condition:

(\textbf{A3}) Let $\textbf{Y}_{i,\theta}^m$ denote a sample (of size $m$) drawn independently from $F_{i,\theta}$. For each player $i$, there exists an algorithm (or, in statistical term, a test) $D_{i,\theta^*}$ that accepts  $\textbf{Y}_{i,\theta}^m$ and $\varepsilon$ as inputs and outputs $0$ or $1$ s.t.~for any $\varepsilon > 0$,
\[
\mathbb{E}\left[D_{i,\theta^*}(\textbf{X}_i^m, \varepsilon)\right] \longrightarrow 0 \quad\quad\textnormal{and}
\]
\[ 
\begin{array}{c}
\sup_{\theta:\|\theta-\theta^*\|\ge\varepsilon} \mathbb{E}\left[ 1 - D_{i,\theta^*}(\textbf{Y}_{i,\theta}^m, \varepsilon) \right] \longrightarrow 0\ .
\end{array}
\]
An algorithm $D_{i,\theta^*}$ with the above property is called a uniformly consistent distinguisher.

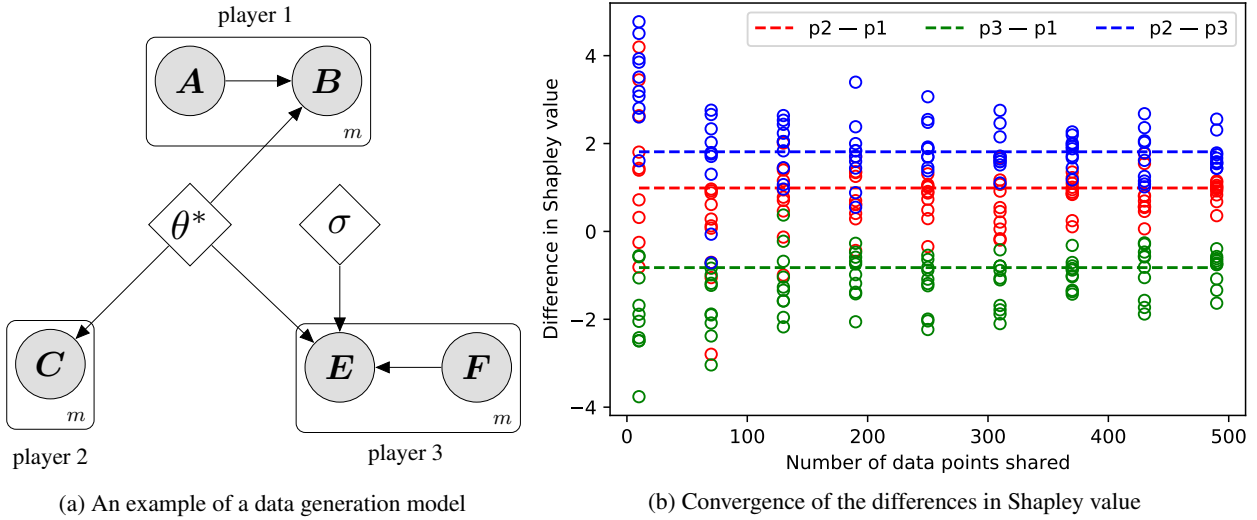
\begin{figure*}
	\centering
	\begin{subfigure}{0.4\textwidth}
		\centering
		\resizebox{\textwidth}{!}{%
			\tikz{%
				\node[det, scale=1.75] (theta*) {$\theta^*$}; %
				\node[obs, scale=1.5, above=of theta*] (X) {$\bm{A}$}; %
				\node[obs, scale=1.5, right=of X] (Y) {$\bm{B}$}; %
				\node[obs, scale=1.5, below left=2cm of theta*] (Z) {$\bm{C}$}; %
				\node[det, scale=1.75, right=of theta*] (sigma) {$\sigma$}; %
				\node[obs, scale=1.5, below=of sigma] (W) {$\bm{E}$}; %
				\node[obs, scale=1.5, right=of W] (U) {$\bm{F}$}; %
				\plate {plate1} {(X) (Y)} {$m$}; %
				\plate {plate2} {(Z)} {$m$}; %
				\plate {plate3} {(U) (W)} {$m$}; %
				\edge {theta*, X} {Y}; %
				\edge {theta*} {Z}; %
				\edge {theta*, U, sigma} {W}; %
				\node[above=0.15cm of X, xshift=1cm] (player1) {player 1}; %
				\node[below=0.9cm of Z, centered] (player2) {player 2}; %
				\node[below=0.5cm of W, xshift=1cm] (player3) {player 3}; %
			}%
		}%
		\caption{An example of a data generation model}
		\label{fig:example:data_generation_model}
	\end{subfigure}%\hspace{1mm}
	\begin{subfigure}{0.58\textwidth}
		\centering%\vspace{-3mm}
		\includegraphics[width=9.4cm]{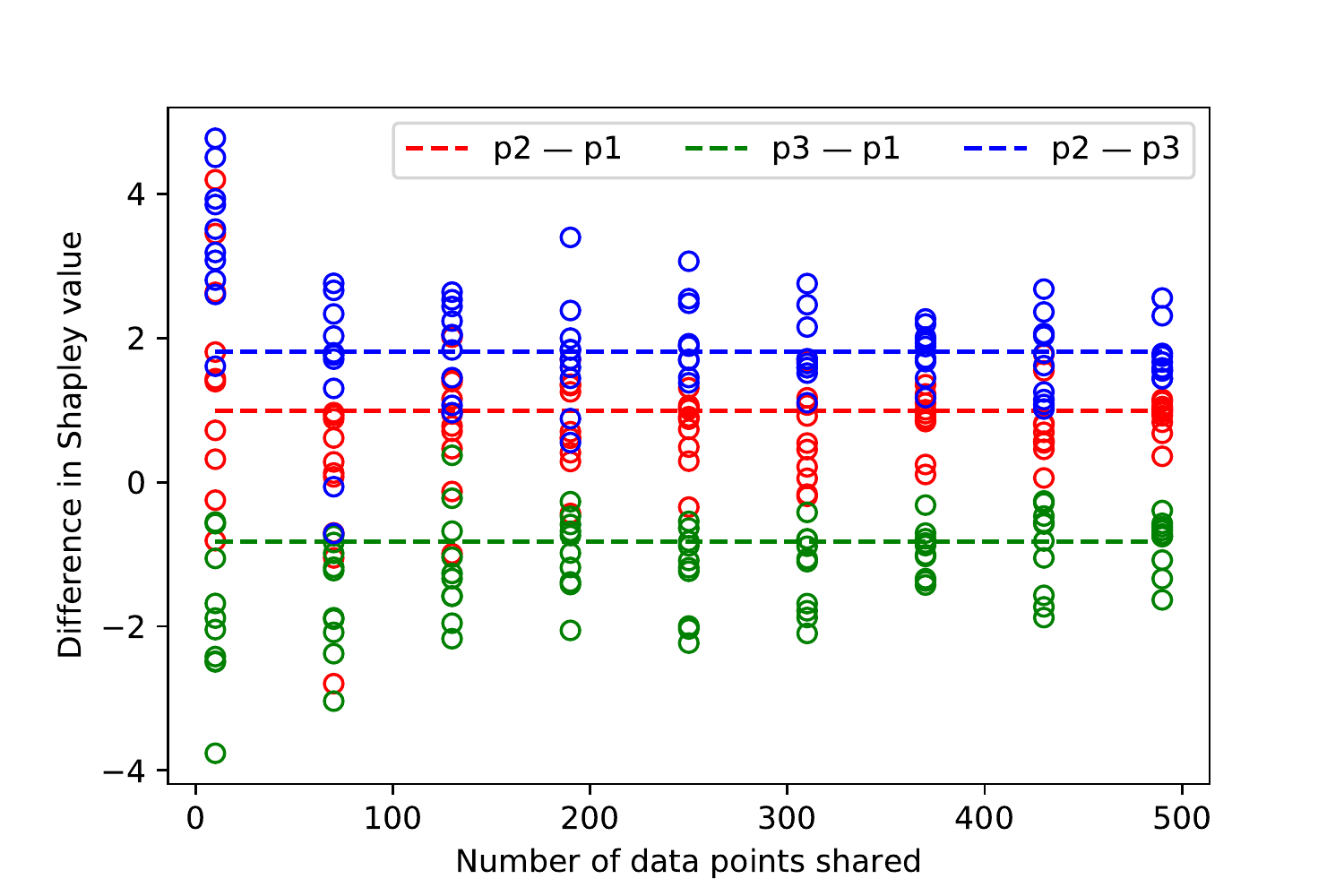}
		\caption{Convergence of the differences in Shapley value}
		\label{fig:example:shapley_value}
	\end{subfigure}
	\caption{An illustrative example of a $3$-player parametric Bayesian learning game.}
	\label{fig:example}
\end{figure*}

As an example, consider the graphical model shown in Fig.~\ref{fig:example:data_generation_model} for three players. In the model, each player observes a disjoint subset of the observable random variables. Each subset depends on the hidden variable $\theta^*$ which is the parameter of interest. The other hidden variable $\sigma$ affects only player $3$. As an illustration, $\theta^*$ may represent, say, the fraction of a population  infected in an epidemic and the three subsets represent different test results (e.g., CT scans, swabs with varying accuracies) performed independently on random samples. In our example, $\theta^*$ is a vector of four real numbers. Player $1$'s data is a linear model (weighted by $\theta^*$) with known noise, player $2$'s data are noisy observations of $\theta^*$, and player $3$' data is a linear model with unknown noise parameter $\sigma$. Note that a player's data may depend on unknowns other than $\theta^*$ although we are interested only in the marginal posterior w.r.t.~$\theta^*$.

Next, we will describe how to perform joint inference. A coalition $S$ is a subset of players who may decide to cooperate to perform inference using their combined data points. With a slight abuse of notations,  $\textbf{X}_S^m := (\textbf{X}_i^m)_{i\in S}$ and $\textbf{X}^m$ denote the combined data points of players in $S$ and that of all the players in the grand coalition, respectively. Bayesian inference is then performed using the joint likelihood $\mathcal{L}_S(\cdot;\textbf{X}_S^m)$ and a commonly agreed prior $\Pi$ whose density is denoted by $\pi$. Then, the joint posterior $\bm{P}_S^m$ is of the density
\[
\begin{array}{rl}
\bm{p}_S^m(\theta) =&\hspace{-2.4mm} \displaystyle \frac{\mathcal{L}_S(\theta; \textbf{X}_S^m)\ \pi(\theta)}{\displaystyle{\int_{\theta'\in\Theta} \mathcal{L}_S(\theta';\textbf{X}_S^m)\ \text{d}\Pi(\theta')}}\ , \vspace{1mm}\\
\mathcal{L}_S(\theta;\textbf{X}_S^m) =&\hspace{-2.4mm} \textstyle\prod_{i\in S}\mathcal{L}_i(\theta;\textbf{X}_i^m)
\end{array}
\]
s.t.~the second equality is due to assumption \textbf{A2}. Note that both the likelihood and the posterior density are random functions whose values depend on the realizations of $\textbf{X}_S^m$. This is also the case for the characteristic function and the Shapley/Banzhaf value defined below. The prior $\Pi$ is assumed to have the following properties: 

(\textbf{A4}) The prior probability measure has a compact support, is absolutely continuous around a neighborhood of $\theta^*$, and has a positive continuous density at $\theta^*$. 

We will now define the characteristic function $\bm{\mathcal{V}}^m$ of the game. Since the function measures the value of a coalition, a natural candidate therefore is to use the information gain on the true parameter via the standard KL divergence of the joint posterior from the common prior: 
\begin{equation}
\bm{\mathcal{V}}^m(S) := \KL{\bm{P}_S^m}{\Pi} := \int_{\Theta}\log\left({\text{d}\bm{P}_S^m}/{\text{d}\Pi}\right) \text{d}\bm{P}_S^m
\label{gatcha}
\end{equation}
which is the expectation (w.r.t.~the posterior) of the log of the Radon-Nikodym derivative of the posterior over the prior. In order for the divergence to be well-defined, $\bm{P}_S^m$ has to be absolutely continuous w.r.t.~$\Pi$. This is always satisfied since in Bayesian inference, any set with zero prior measure will also have zero posterior measure. By definition, the value of the empty coalition is zero. Since a cooperative game is completely defined by its characteristic function, we will use the two terms interchangeably.

As mentioned previously, in cooperative game theory, the two most commonly used measures of contributions are the Shapley and Banzhaf values. Both values are defined as a weighted sum of a player's marginal contributions to all possible coalitions of other players, albeit with different weights. That is, for each player $i$, both values have the form
\[
\textstyle
\phi(i;\bm{\mathcal{V}}^m) :=\sum_{S\subseteq N\setminus\{i\}}  w_S\left[ \bm{\mathcal{V}}^m(S\cup\{i\}) - \bm{\mathcal{V}}^m(S) \right]
\]
where $N$ is the grand coalition and the weights $w_S$ (associated with coalition $S$) for the Shapley and Banzhaf values are, respectively,
\[
|S|!(|N|-|S|-1)!/|N|! \quad\text{and}\quad {1}/{2^{|N|-1}}\ .
\]
The linearity/additivity property holds for both values and is central to deriving the main result of this paper: Let $V$ be a cooperative game. If $V$ can be decomposed additively (i.e., $V = V_1 + \ldots + V_n$), then $\phi(i; V) = \phi(i; V_1) + \ldots + \phi(i, V_n)$.

\section{Convergence Results}

The convergence results are divided into three parts: (a) w.r.t.~the joint posteriors, (b) the characteristic function, and (c) the Shapley/Banzhaf value. For each result, we give its interpretation and the idea of its proof in the main paper; their full proofs are in Appendix~\ref{goodness}. We start with a result on the convergence of the joint posterior $\bm{P}_S^m$ for any nonempty coalition $S$, which generalizes the Bernstein-von Mises theorem to that of joint Bayesian inference. The proof therefore involves checking that the required conditions still hold for the joint inference; see the works of~\citet{lecam2000,vaart1998} for a complete proof of the Bernstein-von Mises theorem.

\begin{theorem}[\textbf{Bernstein-von Mises}]\label{thm:BvM}
	If the regularity conditions \textnormal{\textbf{A1}} to \textnormal{\textbf{A4}} hold, then for any nonempty coalition $S$,
	\[
	\left\lVert \bm{P}_S^m - \mathcal{N}\left( \bm{\hat{\theta}}^m_S, {m}^{-1}\mathcal{I}^{-1}_S \right) \right\rVert_{\textnormal{TV}} \overset{\textnormal{p}}{\longrightarrow} 0
	\]
	where $\bm{\hat{\theta}}_S^m$ is the maximum likelihood estimate w.r.t.~the joint likelihood $\mathcal{L}_S(\,\cdot\,;\textnormal{\textbf{X}}_S^m)$, $\mathcal{I}_S$ is the Fisher information w.r.t.~the same likelihood, and $\lVert\cdot\rVert_{\textnormal{TV}}$ is the total variation distance.
\end{theorem}

Recall that the Fisher information $\mathcal{I}_S$ is the expected value of the squared gradient of the log-likelihood $\log\mathcal{L}_S(\cdot;\textbf{X}_S^m)$ 
and is thus a function of $\theta$. In this work, the Fisher information is always taken to be at the true parameter $\theta^*$:
\[
\mathcal{I}_S := \mathbb{E}\!\left[\! \left(\frac{\partial}{\partial\theta}\log\mathcal{L}_S (\theta^*;\textbf{X}_S^1)\right)\!\! \left(\frac{\partial}{\partial\theta}\log\mathcal{L}_S (\theta^*;\textbf{X}_S^1)\right)^\top \right]
\]
which is not a random variable and does not depend on $m$ since the expectation is w.r.t.~a single data point from each player in $S$. However, it depends on $S$ since different coalitions have different joint likelihoods. Since the expectation of the gradient is zero at $\theta^*$, the Fisher information is also the variance of the gradient of the log-likelihood at $\theta^*$. If the Fisher information is large, then the log-likelihood is steep around $\theta^*$ and hence makes it easy to identify $\theta^*$, and vice versa. This is why Fisher information can be interpreted as a measure of information that a random variable contains about the underlying true parameter. Note also that it is a positive semidefinite matrix.

Theorem \ref{thm:BvM} states that the joint posterior $\bm{P}_S^m$ can be approximated by the distribution 
\begin{equation}
\bm{\hat{P}}_S^m := \mathcal{N}\left(\bm{\hat{\theta}}_S^m,{m}^{-1}\mathcal{I}_S^{-1}\right)
\label{crappybird}
\end{equation}
s.t.~the accuracy improves with an increasing $m$. Furthermore, since $\bm{\hat{\theta}}_S^m \goinp \theta^*$ under conditions \textbf{A1} to \textbf{A3}, $\bm{\hat{P}}_S^m$ (and thus $\bm{P}_S^m$) converges to the degenerate normal distribution that assigns an infinite density to the true parameter. The KL divergence (i.e., characteristic function~\eqref{gatcha}) therefore goes to infinity for any nonempty coalition. This, in turn, causes the Shapley value of any player to go to infinity.

We will now describe the convergence result w.r.t.~the characteristic function. The idea is to approximate the value of $\bm{\mathcal{V}}^m(S)$ by the KL divergence of the approximated joint posterior from the prior. In general, however, the approximation (a normal distribution) is not absolutely continuous w.r.t.~the prior, so the divergence may not be well-defined. We will instead use an extended version of the KL divergence defined as follows: Let $P$ and $R$ be two probability measures defined on a measurable space $(\mathcal{X}, \Omega)$. Then, the \emph{extended KL divergence} of $P$ from $R$ is defined as
\[
\KLE{P}{R} := \int_{\mathcal{X}} \log\left({\text{d}P|_{R}}/{\text{d}R}\right) \text{d}P|_R
\]
where $P|_R$ is the absolutely continuous component of the Lebesgue's decomposition of $P$ w.r.t.~$R$. Note that the measure $P|_R$ is not necessarily probabilistic, i.e., $P|_R(\mathcal{X})$ need not be one. Informally, instead of taking the integral over the whole support of $P$, only the parts that intersect with the support of $R$ are used. Using the extended KL divergence, we can now state the next result:
\begin{lemma}\label{lem:value_convergence}
	Let $\V^m$ be a parametric Bayesian learning game. Under regularity conditions  \textnormal{\textbf{A1}} to \textnormal{\textbf{A4}}, for any nonempty $S$,
	\[
	\Bigl|\, \bm{\mathcal{V}}^m(S) - \KLE{\bm{\hat{P}}^m_S}{\Pi} \,\Bigr| \overset{\textnormal{p}}{\longrightarrow} 0
	\]
	where $\bm{\hat{P}}^m_S$ is previously defined in~\eqref{crappybird}.
	%$\displaystyle{\bm{\hat{P}}^m_S := \mathcal{N}\left(\bm{\hat{\theta}}^m_S, {m}^{-1}\mathcal{I}^{-1}_S\right)}$.
\end{lemma}
Before proceeding, we will first show how this seemingly obvious result may not be trivial to prove. Suppose that for now, in place of KL divergence, we use the Hellinger distance $h$ instead. The reason for this choice is that it satisfies the triangle inequality and is upper bounded by the total variation distance. So, for every $m$,
\[
\begin{array}{lcl}
	\left| h(\bm{P}^m_S, \Pi) - h(\bm{\hat{P}}^m_S, \Pi) \right| & \le & h(\bm{P}^m_S, \bm{\hat{P}}^m_S) \\
	& \le & \left\lVert\bm{P}^m_S - \bm{\hat{P}}^m_S\right\rVert^{1/2}_{\textnormal{TV}} \ .
\end{array}
\]
It follows that for any $\varepsilon > 0$,
\[
\begin{array}{l}
	\Pr\left[\left| h(\bm{P}^m_S, \Pi) - h(\bm{\hat{P}}^m_S, \Pi) \right| > \varepsilon\right]  \vspace{0.5mm}\\
	\le\Pr\left[\left\lVert\bm{P}^m_S - \bm{\hat{P}}^m_S\right\rVert_{\textnormal{TV}} > \varepsilon^2 \right],
\end{array}
\]
which establishes the required convergence in terms of the convergence in total variation distance. Since the latter is true given Theorem~\ref{thm:BvM}, so is the former. On the other hand, the KL divergence neither satisfies the triangle inequality nor is upper bounded by the total variation distance. Instead, we use a triangle-like inequality which establishes its convergence in terms of the convergence in total variation distance, the difference in entropy, and a residual term. The latter two are then shown to converge in probability to zero.

The above result allows us to obtain a closed-form approximation of $\V^m(S)$ for specific priors by analyzing the extended KL divergence. If the latter converges in probability to some fixed function, then so is the former. Next, we present such results for both uniform and normal priors:
\begin{corollary}\label{cor:uniform_prior}
	Let \textnormal{$\bm{\mathcal{V}}^m$} be a parametric Bayesian learning game s.t.~the regularity conditions \textnormal{\textbf{A1}} to \textnormal{\textbf{A4}} hold. If the commonly agreed prior is a uniform distribution $U$ with support $\Theta_U\subseteq\Theta$, then for any nonempty $S$,
	\[
	\left| \bm{\mathcal{V}}^m(S) - \left( \frac{k}{2} \log \left(\frac{m}{2\pi e}\right) + \log\lambda(\Theta_U) + \frac{1}{2}\log|\mathcal{I}_S| \right) \right|
	\]
	converges in probability to $0$ w.r.t.~Lebesgue measure $\lambda$.
\end{corollary}
In condition \textbf{A4}, we assume that the support of the prior is compact, but this is not the case for a normal distribution. However, in practice, the difference between a truncated normal with a large support (centered at the mean) and the standard version is negligible. For clarity of presentation, we state the result informally using the standard normal instead of formally with a truncated one: Let \textnormal{$\bm{\mathcal{V}}^m$} be a parametric Bayesian learning game s.t.~the regularity conditions \textnormal{\textbf{A1}} to \textnormal{\textbf{A4}} hold. If the commonly agreed prior is a normal distribution $\mathcal{N}(\theta_0, \Sigma_0)$, then for any nonempty $S$,
\[
\bm{\mathcal{V}}^m(S) \approx 0.5\ k\log m + \xi + 0.5\log|\mathcal{I}_S|
\]
for a large $m$ where $\xi\hspace{-0.5mm} :=\hspace{-0.5mm} 0.5\hspace{-0.7mm}\left(\hspace{-0.1mm} \left\|\theta_0 - \theta^* \right\|^2_{\Sigma_0^{-1}}\!\! - k + \log|\Sigma_0|\right)$ is a constant that depends on the choice of the prior's parameters, the true parameter, and its dimension $k$.

In both cases above, the characteristic function $\bm{\mathcal{V}}^m$ can be approximately decomposed into a sum of three terms: The first term grows logarithmically in $m$ while the rest are fixed. This tells us that while the characteristic function diverges, it does so very slowly. The second term is a constant that depends mainly on the parameters of the prior. The third term, which is proportional to the log-determinant of the joint Fisher information, is the only one that depends on $S$. If we apply $\phi$ which is additive, then roughly,
\[
\phi(i,\V^m) \approx \phi(i, V_1) + \phi(i, V_2) + \phi(i, S\mapsto 0.5\log|\mathcal{I}_S|)
\]
for a large $m$ and some constant (w.r.t.~$S$) functions $V_1$ and $V_2$. So, if we take the difference between two players, then the first two terms cancel out and we are left with the third. The preceding heuristic argument is made precise by the following result and its proof:
\begin{theorem}\label{thm:shapley_convergence}
	Let $\bm{\mathcal{V}}^m$ be a parametric Bayesian learning game and $V$ be a cooperative game s.t.~$V(S):=0.5\log|\mathcal{I}_S|$ if $S$ is nonempty, and $V(S):=0$ otherwise. If $\phi$ is an additive solution concept, then under regularity conditions \textnormal{\textbf{A1}} to \textnormal{\textbf{A4}} and for a uniformly distributed prior and any players $i,j\in N$,
	\[
	\phi(i;\bm{\mathcal{V}}^m) - \phi(j;\bm{\mathcal{V}}^m) \goinp \phi(i;V) - \phi(j;V)
	\]
	as $m\rightarrow\infty$. We call $V$ a limiting game of $\bm{\mathcal{V}}^m$ w.r.t.~$\phi$.
\end{theorem}
We will now mention several remarks regarding the above result: (a) It does not say that the game $\bm{\mathcal{V}}^m$ converges to $V$. In fact, $\bm{\mathcal{V}}^m$ is divergent, as described previously. Rather, it says that as far as the differences in Shapley value are concerned, the two games are equivalent in the limit. (b) Though a uniform prior is assumed, the result also holds for the normal prior, as argued previously and  shown in our experimental results. (c) The result applies not only to the Shapley/Banzhaf value, but also to any additive solution concept or even, more generally, any linear transformation of the characteristic function. (d) Notice that the function $V$ does not depend on the choice of the prior's parameters. In other words, in the presence of a large quantity of data, the choice of the prior is inconsequential; the relative contributions of the players depend only on their Fisher information. Finally, (e) the result introduces the notion of limiting games which is useful in describing asymptotically equivalent sequences of games. Two games are equivalent (w.r.t.~a solution concept) when there exists a fixed game that is limiting for both.

Coming back to the example in Fig~\ref{fig:example}, in Fig.~\ref{fig:example:shapley_value}, we show the differences in Shapley value as the number $m$ of data points shared increases. For each $m$, we generate $10$ instances of the game by sampling $m$ data points from the observable variables. Bayesian inferences are then performed using a normal prior and the Shapley values are computed using posterior-prior KL divergences. The differences are then plotted and represented by the circles in Fig.~\ref{fig:example:shapley_value}. The dashed lines represent the differences computed based on the game $S\mapsto 0.5\log|\mathcal{I}_S|$. Fig.~\ref{fig:example:shapley_value} shows that as $m$ increases, the sampled differences get increasingly concentrated around the dashed lines, hence aligning with Theorem~\ref{thm:shapley_convergence}.

\section{Shapley-Fair Online Collaborative Learning Framework} \label{sec:shapley-fair-algorithm}

In the previous section, we have shown that under some regularity conditions, the difference in Shapley values converges to a value that is determined solely by the Fisher information. In this section, we exploit this fact to present a $2$-player iterative learning framework that is asymptotically Shapley-fair. In this framework, players do not provide their data all at once but rather at a rate that is determined by the framework. The main idea is to vary the rates s.t.~in the limit, the determinants of both players' Fisher information are the same. This is achieved by bundling several data points together and treating them as a single unit. Consider a $2$-player game with players $1$ and $2$. If we bundle $r$ of player $1$'s data points, then given the i.i.d.~condition of the data points, its new Fisher information is given by
\[
\begin{array}{l}
\displaystyle \text{Var}\left[ \frac{\partial}{\partial\theta} \log\prod_{j=1}^{r} \mathcal{L}_1(\theta^*; \textbf{X}_{1j})\right] \vspace{0.5mm}\\ 
\displaystyle = \sum_{j=1}^{r} \text{Var} \left[ \frac{\partial}{\partial\theta} \log\mathcal{L}_1(\theta^*;\textbf{X}_{1j}) \right] = r\mathcal{I}_1\ .
\end{array}
\]
So, setting $r = (|\mathcal{I}_2\mathcal{I}_1^{-1}|)^{1/k}$ yields $|r_1\mathcal{I}_1| = |\mathcal{I}_2|$. In other words, if player $1$ provides data points at the rate of $r$ w.r.t.~player $2$, then  the difference in Shapley value converges in probability to zero, by Theorem~\ref{thm:shapley_convergence}.

The two main steps in the framework are (a) the estimation of the true parameter which is the main objective of the collaboration, and (b) the estimation of the Fisher information which is used to determine the rates at which data are received. This is done using sample approximation. The framework starts with some data points from both players and repeats the following steps: 
\begin{enumerate}
	\item Use current available data points to compute an estimate $\bar{\theta}$ of the true parameter.
	\item Let $m_i$ denote the number of player $i$'s data points received so far. For $ i=1,2$, use sample approximation to estimate the Fisher information at $\bar{\theta}$:
	\[
	\hat{\mathcal{I}}_i \!=\! \frac{1}{m_i} \! \sum_{j=1}^{m_i} \! \left( \!\!\frac{\partial}{\partial\theta} \log\mathcal{L}_i(\bar{\theta};x_{ij}) \!\!\right) \!\! \left( \!\! \frac{\partial}{\partial\theta} \log\mathcal{L}_i(\bar{\theta};x_{ij}) \!\! \right)^\top .
	\]
	\item Collect $r_1$ and $r_2$ data points from the respective players $1$ and $2$ s.t.~the proportion $m_1+r_1 : m_2+r_2$ of their cumulative data points  is equal to $|\hat{\mathcal{I}}_2|^{1/k} : |\hat{\mathcal{I}}_1|^{1/k}$.
\end{enumerate}
In practice, minimum and maximum constraints can be imposed on $r_1$ and $r_2$ to limit the number of data points collected in each iteration. If the estimator in step $1$ is consistent, then by the law of large numbers and continuous mapping theorem, the sampled Fisher information converges in probability to the true Fisher information. This, together with Theorem~\ref{thm:shapley_convergence}, gives us the following result:
\begin{corollary}\label{thm:algorithm_convergence}
	Suppose that $\bar{\theta}$ in step $1$ is a consistent estimator of the true parameter. Then, as the number of iterations goes to infinity, the difference in Shapley value between the two players converges in probability to zero.
\end{corollary}
Note that the framework does not require any Bayesian inference. The corollary states that \textit{if} in each iteration, Bayesian inferences were performed and the KL divergences were computed, then the difference in Shapley value (based on those divergences) would converge in probability to zero.

\section{Experimental Results}

In this section, our framework is empirically verified by showing that the Shapley values converge over time in several $2$-player scenarios, followed by multi-player scenarios. 
We will investigate parameter estimation, as in Fig.~\ref{fig:example}: parameter estimation for \textit{Bayesian linear regression}~(BLR) on real-world datasets and mean estimation in a learned latent space of a \textit{variational auto-encoder}~(VAE)~\cite{KingmaW13-vae}.

\paragraph{Data, experiments, and implementation details.}
We will investigate parameter estimation for BLR on three real-world datasets: \textit{California housing}~(CaliH) data~\cite{KELLEYPACE1997291-CaliH}, King County house sales prediction~\cite{kingH-dataset}, and age estimation from facial images~\cite{zhifei2017cvpr-FaceA}; and mean estimation in a learned $2$-dimensional latent space of a VAE on the digits $0$ and $1$ of MNIST. Our code is publicly available at \url{https://github.com/XinyiYS/Parametric-Bayesian-Learning-Games}.

\textit{Experiment~1.}
For the CaliH data consisting of $20640$ data points of $8$-dimensional real-valued features, BLR is directly performed on the $6$ real-valued features.\footnote{We drop the latitude and longitude features for CaliH data.} So, the parameter to be estimated is a $6$-dimensional real-valued vector with a standard normal prior. For KingH~(FaceA) consisting of data of higher dimensions, we train a deep~(i.e., \textit{convolutional}) neural network as a feature extractor to extract $10$~($9$)-dimensional real-valued features to perform BLR. So, the parameter to be estimated is a $10$~($9$)-dimensional real-valued vector. Additional details for KingH and FaceA are in  Appendix~\ref{appendix:experiments}.

\textit{Experiment~2.}
This experiment is motivated from density estimation for high-dimensional data being computationally expensive. So, it is often more practical to estimate their density in the latent space and mean estimation is a step in that direction.
Specifically, supposing the data distribution contains several modes/classes, we want to estimate the mean for each mode in the latent space. In this experiment, we use digits $0$ and $1$ from the MNIST dataset for a $2$-mode mean estimation and the mapping to a $2$-dimensional latent space is given by a pre-trained VAE. So, the parameter to be estimated is $4$-dimensional with a standard normal prior. 

\textit{Implementation of framework in Section~\ref{sec:shapley-fair-algorithm}.}
Starting with $m_i$ initial data points for P$i$~(i.e., player $i$), we perform Bayesian inference to obtain the (joint) posterior mean $\bar\theta$~(as an estimate of the true parameter) to be used for approximating the Fisher information $\hat{\mathcal{I}}_i$ which is in turn used to determine the number $r_i$ of data points to collect for the next iteration as $r_i = r_{i^*} |\mathcal{I}_{i^*}\mathcal{I}_{i}^{-1} |^{1/k}$ where $i^* \coloneqq \arg\!\max_{i\in N} |\mathcal{I}_{i}|$ and the number $r_{i^*}$ of data points to collect for the player with the  higher/highest Fisher information follows a preset constant, and
update $m_i \gets m_i + r_i$.
As the Fisher information of player $i$'s data points is less than that of player $i^*$, player $i$ needs to collect more data points than player $i^*$ in the next iteration.
For $2$ players, this specification of $r_i$ ensures that the difference in their Shapley values converges in probability to zero, by Theorem~\ref{thm:shapley_convergence}. 
To calculate the Shapley values $\phi(i;\bm{\mathcal{V}}^{m_i})$ described in Section~\ref{sec:bayesian-learning-game}, we perform additional inferences on the players' \textit{individual} data~(up to the current iteration) to compute the posterior-prior KL divergences $\bm{\mathcal{V}}^{m_i}(S) = \KL{\bm{P}_S^{m_i}}{\Pi}$~\eqref{gatcha}.

\textbf{2-player specification and results.}
In experiment~$1$, P$1$ samples P$1$\_sample\_size data points from a restricted subset~(of size P$1$\_data) of the original data. This simulates that P$1$ can only observe data~(e.g., housing prices) from a certain area or district for constructing its parameter estimate. Two sampling methods are considered: standard uniform sampling~(denoted as iid) and a method based on the statistical leverage scores~\cite{Drineas2012-leverageiid}~(denoted as lvg\_iid) due to its suitability in linear regression. Intuitively, higher P$1$\_data, P$1$\_sample\_size, and lvg\_iid~(vs.~iid) should result in more accurate parameter estimates~(i.e., higher Fisher information) from P$1$.
P$2$ samples noisy observations from a distribution whose mean is estimated based on a proportion~(P$2$\_ratio) of the original data with a fraction~(P$2$\_nan) of missing values. In other words, P$2$'s observation is more direct but defective due to noise and missing values. 
Intuitively, a higher P$2$\_ratio and a lower P$2$\_nan should result in more accurate parameter estimates~(i.e., higher Fisher information) from P$2$.
Specific details are described 
%in supplementary materials 
in Appendix~\ref{appendix:experiments}.

In experiment~$2$, P$1$ and P$2$ have a subset of the original data to sample from and their sizes are, respectively, denoted by P$1$\_data and P$2$\_data. But, their data (respectively, denoted by P$1$\_ratio and P$2$\_ratio) are biased towards one digit or the other. For example, P$1$\_ratio\ $=0.1$ means that P$1$ has $0.1$~($0.9$) probability of sampling the digit $0$~($1$).

\begin{figure}
    \begin{minipage}{0.5\textwidth}
    \hspace{-3mm}
        \centering
        \includegraphics[width=0.45\textwidth]{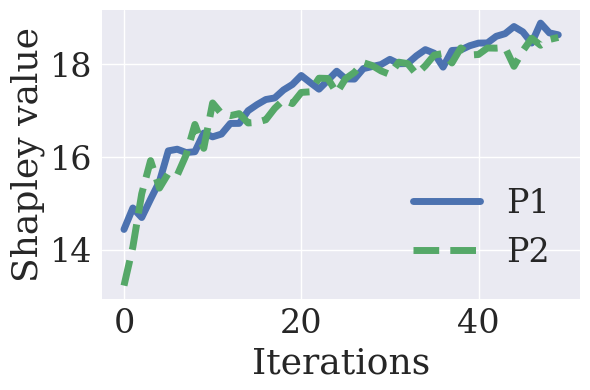}
        \hspace{0.01\textwidth}
        \includegraphics[width=0.45\textwidth]{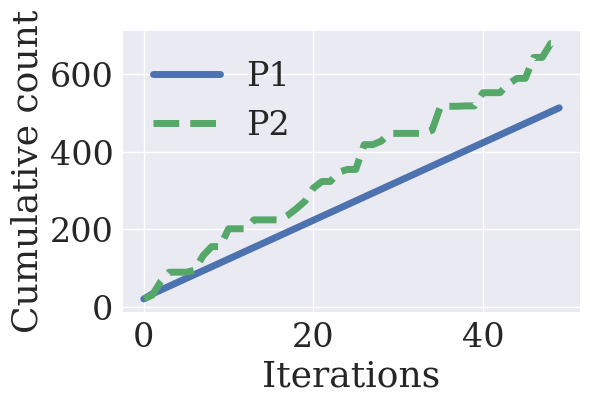}
    \end{minipage}
    \begin{minipage}{0.5\textwidth}
    \hspace{-3mm}
        \centering
        \includegraphics[width=0.45\textwidth]{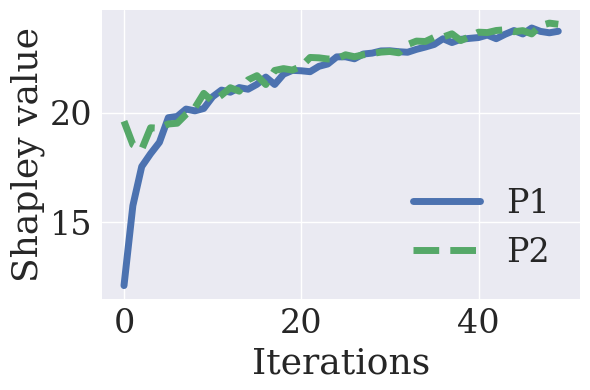}
        \hspace{0.01\textwidth}
        \includegraphics[width=0.45\textwidth]{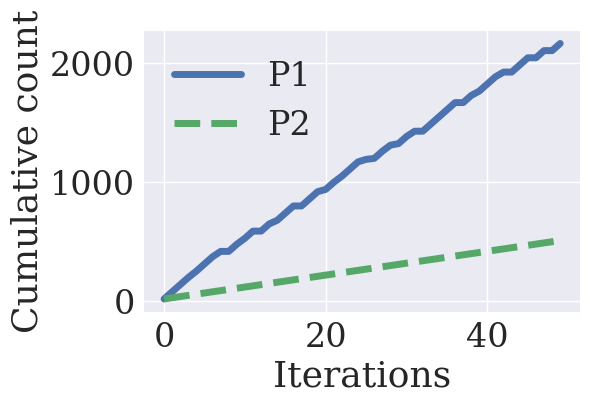}
    \end{minipage}
    \caption{SV and cumulative count vs.~iterations on \textbf{CaliH}~(top) and \textbf{FaceA}~(bottom) for $2$-player scenarios. For CaliH, P$1$\_data=$5000$, P$1$\_sample\_size\ $=500$, P$1$'s sampling is lvg\_iid, P$2$\_ratio\ $=0.01$, and P$2$\_nan\ $=0.4$. For FaceA, P$1$\_data\ $=1000$, P$1$\_sample\_size\ $=100$, P$1$'s sampling is lvg\_iid, P$2$\_ratio\ $=0.1$, and P$2$\_nan\ $=0.1$.}
    \label{fig:CaliH-FaceA}
\end{figure}

The Shapley values and the cumulative counts of shared data points over the iterations are shown in Fig.~\ref{fig:CaliH-FaceA}. To additionally verify convergence, we will analyze the following quantities: (a) statistics of the relative difference $\delta \coloneqq |(\phi_1 - \phi_2) / (\phi_1 + \phi_2)|$, including the lowest, average, and standard deviation, and (b) the number of iterations (denoted by Iter) required for $\delta$ to be smaller than a threshold of $0.1$ for $5$ consecutive iterations. These are computed after a burn-in period of $5$ iterations, as shown in Table~\ref{table:CaliH} for CaliH. Results for other datasets are in Appendix~\ref{appendix:experiments}.

Fig.~\ref{fig:CaliH-FaceA}~(left column) shows that the Shapley values of P$1$ and P$2$ converge over time and verifies our claim. More importantly, our framework adjusts the cumulative counts based on the Fisher information of the players' data points.
The plots on cumulative count in Fig.~\ref{fig:CaliH-FaceA}~(right column) shows that P$2$~(P$1$) needs to consistently collect more data points for CaliH~(FaceA). This is due to the difference in the settings for P$1$ and P$2$'s data: For CaliH, P$1$\_data and  P$1$\_sample\_size are higher and can hence lead to a higher accuracy in P$1$'s parameter estimates. Also, P$2$\_ratio is lower and P$2$\_nan is higher, which can lead to a lower accuracy in P$2$'s parameter estimates. Consequently, for CaliH, P$2$ needs to collect more data points. This shows that our framework identifies the correct relationship between the accuracy of the players' parameter estimates and correspondingly sets the cumulative counts in order to ensure that their Shapley values converge.

Table~\ref{table:CaliH} illustrates this effect in greater detail on CaliH via the statistics on $\delta$ by varying P$1$'s settings. It shows $3$ cases of the difference not diminishing~(i.e., $3$ rows with *), which can be attributed to P$1$'s sample size~(of $10$) being too restrictive to construct accurate parameter estimates to match P$2$'s. This suggests that if the data qualities of P$1$ and P$2$ differ significantly, it may be more difficult~(e.g., more iterations) for their Shapley values to converge. Additional results for other datasets are in  Appendix~\ref{appendix:experiments}.

\begin{table}
\setlength{\tabcolsep}{4pt}
\caption{Statistics on $\delta$ for CaliH with varying P$1$'s contributions.
%~(P$2$'s) on the left~(right). 
Iter denotes the smallest number of iterations for $\delta$ to be smaller than $0.1$ for $5$ consecutive iterations and * denotes cases where it is not satisfied. Also, P$2$\_nan\ $=$\ P$2$\_ratio\ $=0.1$. 
% For right table, P$1$\_data=5000, P$1$\_sample\_size=500, sampling=lvg\_iid.
}
\resizebox{0.48\textwidth}{!}{
\begin{tabular}{rrlrrrl}
\toprule
    P$1$\_data &  P$1$\_sample\_size & sampling &  Lowest &  Average &  StDev &   Iter \\
\midrule
    1000 &        10 &         iid &   0.239 &    0.267 &  0.021 &          50* \\
    1000 &        10 &     lvg\_iid &   0.292 &    0.318 &  0.021 &          50* \\
    1000 &       100 &         iid &   0.090 &    0.104 &  0.011 &           30 \\
    1000 &       100 &     lvg\_iid &   0.092 &    0.105 &  0.011 &           37 \\
    1000 &       500 &         iid &   0.044 &    0.053 &  0.007 &            5 \\
    1000 &       500 &     lvg\_iid &   0.050 &    0.060 &  0.008 &            5 \\
    5000 &        10 &         iid &   0.353 &    0.376 &  0.019 &          50* \\
    5000 &       500 &         iid &   0.094 &    0.105 &  0.008 &           43 \\
    5000 &       500 &     lvg\_iid &   0.072 &    0.083 &  0.008 &            7 \\
\bottomrule
\end{tabular}
}
\label{table:CaliH}
\end{table}

\paragraph{Multi-player extension and results.}
We will first examine the $3$-player scenario in Fig.~\ref{fig:example} (denoted as synthetic): P$1$~(P$3$) samples some noisy linear transformations of the true parameters~(into a $1$-dimensional variable) perturbed with known~(unknown) zero-mean normal noise with standard deviation of $1$~($1.1$), while P$2$ samples noisy observations of the true parameters where the noise is zero-mean isotropic normal with variance $2.5$.
Subsequently, we extend both experiments~$1$ and $2$. For experiment~$1$, we add two players P$3$ and P$4$ where the data sizes of P$3$ and P$1$ are equal but P$1$~(P$3$) adopts the lvg\_iid~(iid) sampling, and P$4$ is defined in the same way as P$2$~(but its actual data are different). 
For experiment~$2$, we add two players P$3$ and P$4$ who have a ratio of $0.5$~(i.e., balanced between digits $0$ and $1$) and data of the same size as P$1$'s and P$2$'s data, respectively.

The Shapley values and the cumulative counts of shared data points over the iterations for synthetic and KingH are shown in Fig.~\ref{fig:multi-synthetic-KingH}. We observe that the Shapley values of the players converge over iterations, which empirically demonstrates that our framework can be generalized to multiple players. For synthetic, we observe that P$1$ and P$3$ need to collect consistently more data points than P$2$ since P$2$ directly samples (noisy) observations of the true parameter. P$1$ and P$3$ do not differ by much as P$3$ maintains and updates a noise estimate dynamically; so, as its noise estimate becomes more accurate, P$3$ becomes closer to P$1$. For KingH, we observe that both P$2$ and P$4$ need to collect consistently more data points than P$1$ and P$3$, possibly because the ratio of missing values is high. The cumulative counts for P$2$ and P$4$ overlap because P$2$ and P$4$ share an identical data specification; similarly, for P$1$ and P$3$, they differ only in the sampling method~(lvg\_iid vs.~iid). Additional experimental results for other datasets are in Appendix~\ref{appendix:experiments}.

\begin{figure}
    \begin{minipage}{0.5\textwidth}
    \hspace{-1mm}
        \centering
        \includegraphics[width=0.45\textwidth]{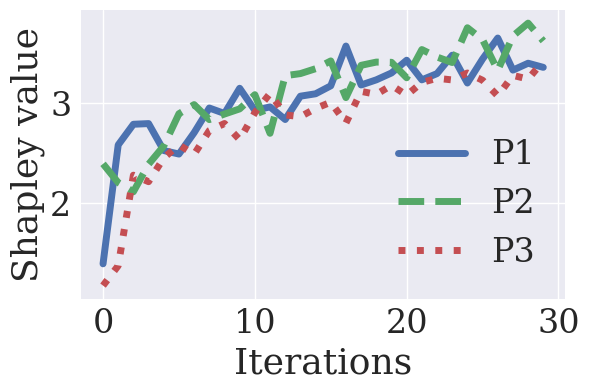}
        \hspace{0.01\textwidth}
        % \hfill
        \includegraphics[width=0.45\textwidth]{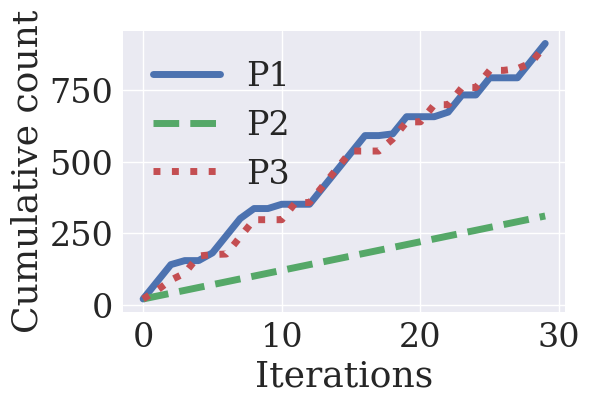}
    \end{minipage}
    % \hfill
    \begin{minipage}{0.5\textwidth}
    \hspace{-3mm}
        \centering
        \includegraphics[width=0.45\textwidth]{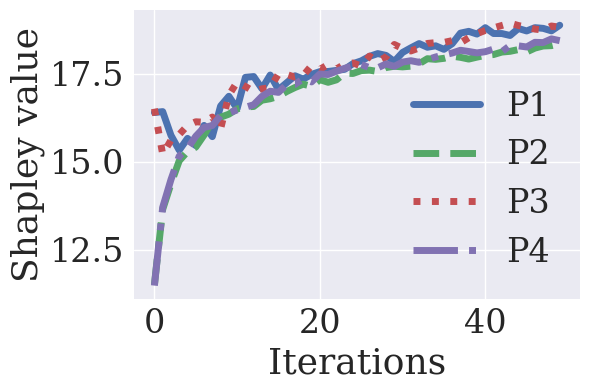}
        \hspace{0.01\textwidth}
        \includegraphics[width=0.45\textwidth]{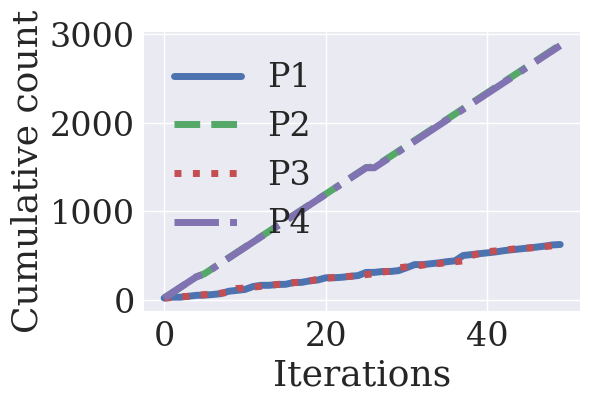}
    \end{minipage}
    \caption{SV and cumulative count vs.~iterations on \textbf{synthetic}~(top) and \textbf{KingH}~(bottom) for multi-player scenarios. 
    For KingH, P$1$/P$3$\_data\ $=1000$, P$1$/P$3$\_sampling\_size\ $=100$, P$2$/P$4$\_ratio\ $=0.1$, and P$2$/P$4$\_nan\ $=0.2$.
    }
    \label{fig:multi-synthetic-KingH}
\end{figure}

\section{Conclusion and Future Work}

In this paper, we study the asymptotic properties of linear solution concepts (in particular, the Shapley value) in parametric Bayesian learning games. We show that in the limit, under some regularity conditions, the differences in Shapley value are determined solely by the players' Fisher information. This allows us to use a consistent Fisher information estimator in an online collaborative framework to achieve asymptotically Shapley-fair data sharing rates. Though the theoretical guarantee is established only for the $2$-player setting, promising empirical results suggest a possibility for a more general result. Extending it to an arbitrary number of players will be the primary focus of our future work. 

Though the assumptions used in this work are mild, their relaxations can be considered in the future works. This includes using a general (possibly improper) prior and dropping the independence assumption between players' data. A particularly interesting direction is the relaxation of the distinguishability property. There may be cases when even with an infinite quantity of data, no single player is able to identify the true parameter. This can occur, for example, when the players are restricted to observe only certain parts of the same sample space. It is only by collaborating that the true parameter can be identified with a greater accuracy. This will provide a stronger motivation for collaborating beyond a faster convergence rate. Finally, the use of different divergences and how that affects the resulting Shapley value would also be considered.

\section*{Acknowledgements}
This research/project is supported by the National Research Foundation Singapore and DSO National Laboratories under the AI Singapore Programme (AISG Award No: AISG$2$-RP-$2020$-$018$).  
Xinyi Xu is supported by the Institute for Infocomm Research of Agency for Science, Technology and Research (A*STAR). 

\bibliography{references}
\bibliographystyle{icml2022}

%%%%%%%%%%%%%%%%%%%%%%%%%%%%%%%%%%%%%%%%%%%%%%%%%%%%%%%%%%%%%%%%%%%%%%%%%%%%%%%
%%%%%%%%%%%%%%%%%%%%%%%%%%%%%%%%%%%%%%%%%%%%%%%%%%%%%%%%%%%%%%%%%%%%%%%%%%%%%%%
% APPENDIX
%%%%%%%%%%%%%%%%%%%%%%%%%%%%%%%%%%%%%%%%%%%%%%%%%%%%%%%%%%%%%%%%%%%%%%%%%%%%%%%
%%%%%%%%%%%%%%%%%%%%%%%%%%%%%%%%%%%%%%%%%%%%%%%%%%%%%%%%%%%%%%%%%%%%%%%%%%%%%%%
\newpage
\appendix
\onecolumn

% Supplementary Materials
\newpage
\appendix

\section{Proofs of Theoretical Results}
\label{goodness}

Throughout, the measurable space is taken to be $\mathbb{R}^k$ equipped with the standard Borel sigma algebra. Unless otherwise stated, absolute continuity is w.r.t.~the Lebesgue measure which is denoted with $\lambda$. We will start with the following claim. The terms evaluated in this claim often appear as residuals in later parts and here, they are shown to converge in probability either to zero or one.

\begin{claim}\label{claim:residuals}
	The following are some properties of the approximate joint posterior $\bm{\hat{P}}_S^m$ in relation to the  support of the prior $\Theta_\Pi$ and its complement $\Theta_\Pi^c$:
	\begin{itemize}
		\item[i.] $\displaystyle{\int_{\Theta_\Pi^c} \text{\emph{d}}\bm{\hat{P}}_S^m \goinp 0}\quad$ and $\quad\displaystyle{\int_{\Theta_\Pi} \text{\emph{d}}\bm{\hat{P}}_S^m \goinp 1}$,
		\item[ii.] $\displaystyle{\int_{\Theta_\Pi^c}} \log \left( \frac{\text{\emph{d}}\bm{\hat{P}}_S^m}{\text{\emph{d}}\lambda} \right) \text{\emph{d}}\bm{\hat{P}}_S^m \goinp 0$.
	\end{itemize} 
\end{claim}
\begin{proof} (i) The density of $\displaystyle{\bm{\hat{P}}_S^m := \mathcal{N}\left(\bm{\hat{\theta}}_S^m, \frac{1}{m}\mathcal{I}_S^{-1}\right)}$, denoted by $\bm{\hat{p}}_S^m(\theta)$, is given by
\begin{align}
	\frac{\displaystyle\exp\left\{ -\frac{1}{2}(\theta - \bm{\hat{\theta}}_S^m)^\top \left(\frac{1}{m}\mathcal{I}_S^{-1}\right)^{-1} (\theta - \bm{\hat{\theta}}_S^m) \right\}}{\displaystyle\sqrt{(2\pi)^k \left|\frac{1}{m}\mathcal{I}_S^{-1}\right|}} & = \frac{\displaystyle m^{k/2} \exp\left\{ -\frac{m}{2}(\theta - \bm{\hat{\theta}}_S^m)^\top \mathcal{I}_S (\theta - \bm{\hat{\theta}}_S^m) \right\}}{\displaystyle\sqrt{(2\pi)^k\left|\mathcal{I}_S^{-1}\right|}} \nonumber\\
	& = C \sqrt{\frac{m^{k}}{\exp({m\cdot\bm{Z}})}} \nonumber
\end{align}
for some constant $C>0$ and nonnegative random variables $\bm{Z}:= (\theta - \bm{\hat{\theta}}_S^m)^\top \mathcal{I}_S (\theta - \bm{\hat{\theta}}_S^m)$. Since $\bm{\hat{\theta}}_S^m \goinp \theta^*$, $\bm{Z} \goinp 0$ if $\theta = \theta^*$, and $\bm{Z}\goinp a$ for some $a>0$ otherwise. This means that for any measurable set $A$ s.t.~$\theta^*\notin A$, the function $\bm{\hat{p}}_S^m$ on $A$ converges pointwise to the zero function, and thus
\[
\bm{\hat{P}}_S^m(A) = \int_A \bm{\hat{p}}_S^m \text{d}\lambda \goinp \int_A 0\, \text{d}\lambda = 0
\]
by the dominated convergence theorem. By Assumption \textbf{A4}, $\theta^*\notin\Theta_\Pi^c$, and therefore
\[
\int_{\Theta_\Pi^c} \text{d}\bm{\hat{P}}_S^m \goinp 0 \quad \text{and} \quad \int_{\Theta_\Pi} \text{d}\bm{\hat{P}}_S^m = 1 - \int_{\Theta_\Pi^c} \text{d}\bm{\hat{P}}_S^m \goinp 1\ .
\]

(ii) Note that ${\lim_{x\to 0^+}} x \log x = \lim_{x\to 0^+} (\log x)/(1/x) \overset{\text{L'H}}{=} \lim_{x\to 0^+}({1/x})/({-1/x^2}) = \lim_{x\to 0^+} -x = 0$. From (i), $\bm{\hat{p}}_S^m$ (on the set $\Theta_\Pi^c$) converges pointwise to the zero function and therefore, so is $\bm{\hat{p}}_S^m\log\bm{\hat{p}}_S^m$. Again, by the dominated convergence theorem, we have the required result.
\end{proof}

\subsection*{Theorem \ref{thm:BvM} (Bernstein-von Mises)}

The conditions needed for Bernstein-von Mises to hold in a coalitional setting are as follows: (i) Twice differentiability of $\mathcal{L}_S$ w.r.t.~$\theta$: This is fulfilled since by Assumption \textbf{A2}, the joint likelihood is the product of individual likelihoods, the individual likelihoods are twice differentiable (Assumption \textbf{A1}), and the products of differentiable functions are differentiable. Note that sometimes a weaker condition is given, i.e., the function is differentiable in the quadratic mean at $\theta^*$ (see, for example, thw work of~\citet{vaart1998}). For our purpose, however, the slightly stronger yet simpler condition suffices. (ii) The combined data are distinguishable from the ones generated using a parameter other than $\theta^*$: This is satisfied as long as one player's data is distinguishable (Assumption \textbf{A3}) since its distinguisher can be used as the distinguisher for the combined data. Finally, (iii) the prior satisfies the conditions given in Assumption \textbf{A4}. In addition, note also that by the independence assumption, 
\[
	\mathcal{I_S} = \textnormal{Var}\left[ \frac{\partial}{\partial\theta}\log\mathcal{L}_S(\theta^*;\textbf{X}_S^1) \right] = \sum_{i\in S} \textnormal{Var} \left[ \frac{\partial}{\partial\theta} \log \mathcal{L}_i(\theta^*;\textbf{X}_{i1}) \right] = \sum_{i\in S} \mathcal{I}_i\ .
\]

\subsection*{Lemma \ref{lem:value_convergence}}

We start with a triangle-like inequality for the extended KL divergence:
\begin{claim*}
	Let $F$, $G$, and $K$ be three probability measures on $\Theta$ s.t.~$F\ll K$, while $G$ may not be absolutely continuous w.r.t.~$K$. Let $\Theta_K$ denote the support of $K$ which is compact. Then, there exists a constant $c>0$ s.t.
	\[
	\Bigl|\, \KL{F}{K} - \KLE{G}{K} \,\Bigr| \le c\cdot\TV{F}{G} + \Bigl| H(F) - H(G) \Bigr| + \left| \int_{\Theta_K^c} \log\left(\frac{\text{\emph{d}}G}{\text{\emph{d}}\lambda}\right) \text{\emph{d}}G \right|
	\]
	where $\Theta_K^c$ denotes the complement of $\Theta_K$.
\end{claim*}
\begin{proof}
	Let $f$, $g$, and $k$ be the densities of $F$, $G$, and $K$, respectively. Then, by definition,
	\[
	\Bigl|\, \KL{F}{K} - \KLE{G}{K} \,\Bigr| = \left| \int_{\theta\in\Theta_K}f(\theta)\log\left( \frac{f(\theta)}{k(\theta)} \right) \text{d}\theta - \int_{\theta\in\Theta_K} g(\theta)\log\left(\frac{g(\theta)}{k(\theta)}\right) \text{d}\theta \right| .
	\] 
	After rearranging the terms,
	\[
	\left| \int_{\theta\in\Theta_K} [g(\theta) - f(\theta)] \cdot \log k(\theta)\ \text{d}\theta + \int_{\theta\in\Theta_K}f(\theta)\log f(\theta)\ \text{d}\theta - \int_{\theta\in\Theta_K} g(\theta)\log g(\theta)\ \text{d}\theta \right|
	\] 
	which is upper bounded by
	\[
	\int_{\theta\in\Theta_K} \Bigl| g(\theta) - f(\theta) \Bigr| \cdot \Bigl| \log k(\theta) \Bigr| \ \text{d}\theta + \left| -H(F) - \int_{\theta\in\Theta_K} g(\theta)\log g(\theta)\ \text{d}\theta \right|.
	\]
	For the first term, since $\Theta_K$ is compact, the logarithmic part is bounded and
	\begin{eqnarray}
		\int_{\theta\in\Theta_K} \Bigl| g(\theta) - f(\theta) \Bigr| \cdot \Bigl| \log k(\theta) \Bigr| \ \text{d}\theta & \le & \int_{\theta\in\Theta_K} \Bigl| g(\theta) - f(\theta) \Bigr| \cdot \max_{\theta'\in\Theta_K} \Bigl| \log k(\theta') \Bigr| \ \text{d}\theta \nonumber\\
		& \le & \max_{\theta'\in\Theta_K} \Bigl| \log k(\theta') \Bigr| \cdot \int_{\theta\in\Theta} \Bigl| g(\theta) - f(\theta) \Bigr| \ \text{d}\theta \nonumber\\
		& = & \max_{\theta'\in\Theta_K} \Bigl| \log k(\theta') \Bigr| \cdot \left\|\, g - f \,\right\|_{L^1} \nonumber\\
		& = & \max_{\theta'\in\Theta_K} \Bigl| \log k(\theta') \Bigr| \cdot 2 \, \TV{G}{F} \ . \nonumber
	\end{eqnarray}
	For the second term, 
	\begin{eqnarray}
		\left| -H(F) - \int_{\theta\in\Theta_K} g(\theta)\log g(\theta)\ \text{d}\theta \,\right| & = & \left| -H(F) + H(G) + \int_{\theta\in\Theta_K^c} g(\theta)\log g(\theta)\ \text{d}\theta \,\right| \nonumber\\
		& \le & \Bigl| H(F) - H(G) \Bigr| + \left| \int_{\theta\in\Theta_K^c} g(\theta)\log g(\theta)\ \text{d}\theta \right|, \nonumber
	\end{eqnarray}
	as required.
\end{proof}

The claim above shows that the difference between $\KL{F}{K}$ and $\KLE{G}{K}$ can be bounded in terms of the total variation distance between $F$ and $G$, their difference in entropy, and a residual term. Note that when $G$ is also absolutely continuous w.r.t.~$K$, the extended KL divergence reduces to the standard definition and the residual term disappears. By applying the claim, for any $m$ and $\varepsilon > 0$, there exists a constant $c>0$ s.t.
\[
\begin{array}{l}
	\displaystyle\Pr\left[\Bigl|\, \KL{\bm{P}^m_S}{\Pi} - \KLE{\bm{\hat{P}}^m_S}{\Pi} \,\Bigr| > \varepsilon \right] \nonumber\vspace{0.5mm}\\
	\displaystyle\le \Pr\left[ c \cdot \TV{\bm{P}^m_S}{\bm{\hat{P}}^m_S} + \left| H(\bm{P}^m_S) - H(\bm{\hat{P}}^m_S) \right| + \left| \int_{\Theta_\Pi^c} \log\left(\frac{\text{d}\bm{\hat{P}}^m_S}{\text{d}\lambda}\right) \text{d}\bm{\hat{P}}^m_S \right| > \varepsilon \right] \nonumber\vspace{0.5mm}\\
	\displaystyle\le \Pr\left[ \TV{\bm{P}^m_S}{\bm{\hat{P}}^m_S} > \varepsilon/(3c) \right] + \Pr\left[ \left| H(\bm{P}^m_S) - H(\bm{\hat{P}}^m_S) \right| > \varepsilon/3 \right] + \Pr\left[ \left| \int_{\Theta_\Pi^c} \log\left(\frac{\text{d}\bm{\hat{P}}^m_S}{\text{d}\lambda}\right) \text{d}\bm{\hat{P}}^m_S \right| > \varepsilon/3 \right]. \nonumber
\end{array}
\]
Therefore, if the three terms converge in probability to zero, then so is the difference in extended KL divergence. The convergence of the first term (i.e., in total variation distance) is given by Theorem $\ref{thm:BvM}$ and that of the residual term is given by Claim \ref{claim:residuals}(ii). The convergence in entropy is established next.

The convergence properties of (differential) entropy have been studied extensively in the information-theoretic literature with different conditions derived for the convergence to hold in terms of convergence in various measures (e.g., $L^1$-distance, KL divergence, pointwise convergence). Here, we adapted the approach found in \cite{ghourchian2017} that is well-suited to our setting.

\begin{claim*}
	Let $F$ and $G$ be two probability measures where the support $\Theta_F$ of $F$ is compact while $G$ is absolutely continuous. Then,
	\[
		|H(F) - H(G)| \,\,\le\,\, \delta \cdot\log \frac{\lambda(\Theta_F)}{\delta} + \left|\int_{\Theta_F^c}\!\! \log\left(\frac{\text{\emph{d}}G}{\text{\emph{d}}\lambda}\right) \text{\emph{d}}G\right|
	\]
	for some $0\le\delta\le2\TV{F}{G}\ .$
\end{claim*}
\begin{proof}
Let $f$ and $g$ be the densities of $F$ and $G$, respectively. Furthermore, let $M$ be defined as
\[
	M := \max\left\{\sup_{\theta\in\Theta_F} f(\theta) \,,\, \sup_{\theta\in\mathbb{R}^k} g(\theta)\right\}.
\]
We first apply a change of variables to $\bm{Y}_F\sim F$ and $\bm{Y}_G\sim G$ by scaling them to obtain the random variables $2M\bm{Y}_F$ and $2M\bm{Y}_G$, respectively. Let $\hat{F}$ and $\hat{G}$ be the new measures, and $\hat{f}$ and $\hat{g}$ be their densities, respectively. Then, 
\[\begin{array}{ll}
	\displaystyle{\hat{f}(\theta) = \frac{1}{2M}\, f\! \left(\frac{\theta}{2M}\right)} & \quad\text{with support } \Theta_{\hat{F}} := \{2M\theta\}_{\theta\in\Theta_F}\ , \\
	\displaystyle{\hat{g}(\theta) = \frac{1}{2M} \,g\! \left(\frac{\theta}{2M}\right)} & \quad\text{with support } \mathbb{R}^k\ .
\end{array}\]
By scaling, we make the densities' values lie in between $0$ and $1/2$ while keeping both their $L^1$-distance and difference in entropy the same (since we are scaling both by the same scalar). By letting $u:=\theta/(2M) \Rightarrow \text{d}\theta = 2M\, \text{d}u$, 
\[
	\int_{\theta\in\mathbb{R}^k}\! |\hat{f}(\theta) - \hat{g}(\theta)|\ \text{d}\theta = \frac{1}{2M}\int_{\theta\in\mathbb{R}^k}\! \left|f\left(\frac{\theta}{2M}\right) - g\left(\frac{\theta}{2M}\right)\right| \text{d}\theta = \int_{u\in\mathbb{R}^k}|f(u) - g(u)|\ \text{d}u
\]
and
\[
	\left|H(\hat{F}) - H(\hat{G})\right| = |H(F) + \log(2M) - H(G) - \log(2M)| = |H(F) - H(G)|\ .
\]
The purpose of scaling the densities to interval $[0,1/2]$ is so that we can apply the following inequality later: If $x,y\in[0,1/2]$, then
\begin{equation}\label{useful_inequality}
	|x\log x - y\log y| \le |x-y|\log\frac{1}{|x-y|}\ .
\end{equation}
The proof of the inequality can be found in Theorem $17.3.3$ (Equation $17.27$) of \citet{cover2006}. Note that here we are indeed trying to generalize Theorem $17.3.3$ to the continuous setting. We now construct another probability measure $Z$ on $\Theta_F$ with the following density:
\[
	z(\theta) := \left\{\begin{array}{ll} 
	\displaystyle{\frac{|\hat{f}(\theta) - \hat{g}(\theta)|}{\delta}} & \hspace{1cm} \text{if}\ \theta\in\Theta_F\ , \\
	0 & \hspace{1cm} \text{otherwise}\ ;
	\end{array}\right.
\]
where $\displaystyle{\delta:=\int_{\theta\in\Theta_F}\!\!|\hat{f}(\theta) - \hat{g}(\theta)| \ \text{d}\theta}$ is the normalizing factor. Note that
\[
	\delta \le \|\hat{f} - \hat{g}\|_{L^1} = \|f - g\|_{L^1} = 2\TV{F}{G}\ .
\]
Putting all the above together, 
\begin{align}
	|H(F) - H(G)| & = \left| \int_{\Theta_F}f(\theta)\log f(\theta) \ \text{d}\theta - \int_{\Theta_F} g(\theta)\log g(\theta) \ \text{d}\theta - \int_{\Theta_F^c} g(\theta)\log g(\theta)\ \text{d}\theta \right| \nonumber\\
	& \le \int_{\Theta_F} \Bigl| \hat{f}(\theta)\log \hat{f}(\theta) - \hat{g}(\theta)\log \hat{g}(\theta) \Bigr| \ \text{d}\theta + \left| \int_{\Theta_F^c} g(\theta)\log g(\theta)\ \text{d}\theta \right| =: R \nonumber\\
	& \le \int_{\Theta_F} |\hat{f}(\theta)-\hat{g}(\theta)|\cdot\log\frac{1}{|\hat{f}(\theta)-\hat{g}(\theta)|} \ \text{d}\theta + R \hspace{1cm} \text{(by Inequality~\ref{useful_inequality})} \nonumber\\
	& = \int_{\Theta_F} \delta z(\theta)\cdot\log\frac{1}{\delta z(\theta)} \ \text{d}\theta + R \hspace{1cm} \text{(by Definition of $z$)} \nonumber\\
	& = -\delta \log \delta \int_{\Theta_F} z(\theta) \ \text{d}\theta - \delta\int_{\Theta_F} z(\theta) \log z(\theta) \ \text{d}\theta + R\nonumber\\
	& = -\delta\log\delta + \delta H(Z) + R \phantom{\int}\nonumber\\
	& \le \delta\log\frac{\lambda(\Theta_F)}{\delta} + R \ ,\nonumber 
\end{align}
as required.
\end{proof}
By applying the claim to our setting,
\[
|H(\bm{P}_S^m) - H(\bm{\hat{P}}_S^m)| \le \bm{\delta}^m \log\frac{\lambda(\Theta_\Pi))}{\bm{\delta}^m} + \left|\int_{\Theta_\Pi^c}\log\left(\frac{\text{d}\bm{\hat{P}}_S^m}{\text{d}\lambda}\right) \text{d}\bm{\hat{P}}_S^m\right|
\]
which converges in probability to zero since $\bm{\delta}^m \le 2\TV{\bm{P}_S^m}{\bm{\hat{P}}_S^m} \goinp 0$, by Theorem~\ref{thm:BvM}, and the residual term as well, by Claim~\ref{claim:residuals}(ii). This completes the proof.

\subsection*{Corollary \ref{cor:uniform_prior}}

By applying the definition of the extended KL divergence, 
\begin{eqnarray}
	\KLE{\bm{\hat{P}}^m_S}{U(\Theta)} & = & \int_{\Theta_U} \log\left(\frac{\bm{\hat{p}}^m_S}{1/\lambda(\Theta_U)}\right) \text{d}\bm{\hat{P}}^m_S \nonumber\\
	& = & \int_{\Theta_U} \log(\bm{\hat{p}}^m_S)\  \text{d}\bm{\hat{P}}^m_S + \log\lambda(\Theta_U)\int_{\Theta_U}\ \text{d}\bm{\hat{P}}^m_S \nonumber\\
	& = & -H(\bm{\hat{P}}^m_S) - \int_{\Theta_U^c} \log(\bm{\hat{p}}^m_S)\  \text{d}\bm{\hat{P}}^m_S + \log\lambda(\Theta_U)\int_{\Theta_U}\ \text{d}\bm{\hat{P}}^m_S \nonumber\\
	& = & \frac{k}{2} \log \left(\frac{m}{2\pi e}\right) + \frac{1}{2}\log|\mathcal{I}_S| - \int_{\Theta_U^c} \log(\bm{\hat{p}}^m_S) \  \text{d}\bm{\hat{P}}^m_S +  \log\lambda(\Theta_U)\int_{\Theta_U} \  \text{d}\bm{\hat{P}}^m_S\ . \nonumber
\end{eqnarray}
From Claim~\ref{claim:residuals}, 
\[
\int_{\Theta_U^c} \log(\bm{\hat{p}}^m_S)\  \text{d}\bm{\hat{P}}^m_S \goinp 0 \quad\text{and}\quad \int_{\Theta_U} \  \text{d}\bm{\hat{P}}^m_S \goinp 1
\]
which establish that 
\[
\left| \KLE{\bm{\hat{P}}^m_S}{U(\Theta)} - \left( \frac{k}{2} \log \left(\frac{m}{2\pi e}\right) + \log\lambda(\Theta_U) + \frac{1}{2}\log|\mathcal{I}_S| \right) \right| \goinp 0.
\]
Together with Lemma \ref{lem:value_convergence}, this gives us the required result.

\subsection*{Theorem \ref{thm:shapley_convergence}}

The characteristic function of a cooperative game is a mapping from $2^N$ to $\mathbb{R}$. If we arrange the values of all coalitions in a vector, then we get a $2^{|N|}$-dimensional real-valued vector. In other words, the set of all possible characteristic functions is the  $\mathbb{R}^{2^{|N|}-1}$ space; the dimension is one less because the value of an empty coalition is always zero. Note that in general, negative values are permitted. By definition, the Shapley/Banzhaf value $\phi$ is a linear transformation from the  $\mathbb{R}^{2^{|N|}-1}$ space to the  $\mathbb{R}^{|N|}$ space and is therefore continuous. Now, define the function $\Delta_{ij}(v)$ for any players $i\neq j$ as
\begin{align}
\Delta_{ij} & : \mathbb{R}^{2^{|N|}-1} \to \mathbb{R} \ ,\nonumber\\
v & \mapsto \phi(i,v) - \phi(j,v)\ . \nonumber
\end{align}
In other words, given a characteristic function $v$, $\Delta_{ij}(v)$ computes the difference in Shapley value between players $i$ and $j$. Note that $\Delta_{ij}$ is also a linear transform, by definition, and is thus continuous. So, by the \emph{continuous mapping theorem} (CMT), if there is a parametric Bayesian learning game $\V^m$ s.t.~$\V^m\goinp v$ for some game $v$, then $\Delta_{ij}(\V^m)\goinp \Delta_{ij}(v)$. From Corollary~\ref{cor:uniform_prior}, $\V^m$ converges in probability to a game that can be decomposed into three games $V_1 + V_2 + V$. $V_1$ and $V_2$ are vectors where all their entries are the same. By definition, if $v$ is a vector with the same entries, then $\Delta_{ij}(v) = 0$. Since $V$ is the game $S\mapsto 0.5\log|\mathcal{I}_S|$, 
\begin{align}
	\Delta_{ij}(\V^m) & = \Delta_{ij}(\V^m) - \Delta_{ij}(V_1) - \Delta_{ij}(V_2) && \text{(by Definitions of $\Delta_{ij}(V_1)$ and $\Delta_{ij}(V_2)$)} \nonumber\\ 
	& = \Delta_{ij}(\V^m - (V_1 + V_2)) && \text{(by linearity of $\Delta_{ij}$)} \nonumber\\
	& \goinp \Delta_{ij}(V) && \text{(by Corollary \ref{cor:uniform_prior} \& CMT)}\ , \nonumber
\end{align}
as required.

\subsection*{Corollary~\ref{thm:algorithm_convergence}}

Note that since there are two players, the characteristic function can be represented by a vector of three values. Consider a $2$-player cooperative game where the characteristic function is given by $V:=(0.5\log|\mathcal{I}_1|, 0.5\log|\mathcal{I}_2|, 0.5\log|\mathcal{I}_{12}|)$. Suppose further that the determinant of both players' Fisher information is the same, i.e., $|\mathcal{I}_1| = |\mathcal{I}_2|$. By the symmetric property, if the marginal contributions of the players to any other coalition are the same, then their Shapley values are the same. The only other coalition is the empty coalition and
\[
V(\{1\}) - V(\{\}) = \frac{1}{2}\log|\mathcal{I}_1| = \frac{1}{2}\log|\mathcal{I}_2| = V(\{2\}) - V(\{\})\ . 
\]
So, $\phi(1,V) - \phi(2,V) = 0$. So, if $\V^m$ is a parametric Bayesian games for two players s.t.~$|\mathcal{I}_1| = |\mathcal{I}_2|$, then 
\begin{align}
\phi(1, \V^m) - \phi(2, \V^m) & \goinp \phi(1,V) - \phi(2,V) && \text{(by Theorem \ref{thm:shapley_convergence})} \nonumber\\
& = 0\ . \nonumber
\end{align}

In general, in the case of $|\mathcal{I}_1| \neq |\mathcal{I}_2|$, if we bundle $r_1$ of player $1$'s data points together and treat them as a single point, we get a new Fisher information $r_1\mathcal{I}_1$ (similarly, $r_2\mathcal{I}_2$ for player $2$). Therefore, if we set $r_1 : r_2 = {|\mathcal{I}_2|}^{1/k} : {|\mathcal{I}_1|}^{1/k}$, then
\[
|r_1\mathcal{I}_1| = r_1^k|\mathcal{I}_1| = \left(\sqrt[k]{\frac{|\mathcal{I}_2|}{|\mathcal{I}_1|}}\ r_2\right)^k \cdot |\mathcal{I}_1| = r_2^k|\mathcal{I}_2| = |r_2\mathcal{I}_2|\ .
\]
In other words, the new Fisher information are the same. This is equivalent to having the total number of points shared by the players be proportional to $r_1 : r_2$, which implies that if in step $3$ of the framework, both $\hat{\mathcal{I}}_1 \goinp \mathcal{I}_1$ and $\hat{\mathcal{I}}_2 \goinp \mathcal{I}_2$, then $|r_1\mathcal{I}_1| - |r_2\mathcal{I}_2| \goinp 0$ and the conclusion follows. In step $2$, as the number of iterations increases, we have $\hat{\mathcal{I}}_1 \goinp \mathcal{I}_1(\bar{\theta})$ and $\hat{\mathcal{I}}_2\goinp\mathcal{I}_2(\bar{\theta})$, by the weak law of large numbers. Also, since we assume that $\bar{\theta}$ is a consistent estimator, by the continuous mapping theorem, $\hat{\mathcal{I}}_1\goinp\mathcal{I}_1$ and $\hat{\mathcal{I}}_2\goinp\mathcal{I}_2$, as required.

\section{Additional Experimental Details and Results} \label{appendix:experiments}

\subsection*{Sharing Data Points in Learned Features with Neural Networks}
King county housing sales prediction~(KingH)\footnote{\url{https://www.kaggle.com/harlfoxem/housesalesprediction}}  contains over $21$K data of $21$ features. Age prediction from facial images~(FaceA)~\cite{zhifei2017cvpr-FaceA} contains over $20$K face images of various sizes for people with ages between $0$ and $116$.
For KingH, we train a neural network with two fully-connected layers of $64$ and $10$ hidden units with \emph{rectified linear units} (ReLU) as the activation function. 
For FaceA, we train a convolutional neural network with $3$ convolutional layers~(with thefirst $2$ followed by batch normalization and max pooling) and $3$ fully-connected layers~(respectively, with $1024$, $64$, $10$ hidden units).
Then, we pass the original data through the respective trained neural network models to construct the features obtained after the last fully-connected layer of $10$ hidden units to construct features in the learned space of $10$ dimensions. We find that there is a redundant extracted feature for the FaceA dataset~(containing all zeros), so it is dropped and the dimension for the extracted features is $9$.

\textbf{Details of experiment~1 on BLR.}
We will first describe how the real-world datasets are used by players P$1$ and P$2$. P$1$ has a randomly selected subset (of size P$1$\_data) from the original dataset. P$2$ has a randomly selected subset which is a fraction (P$2$\_ratio) of the original dataset. Furthermore, P$2$'s data has a proportion~(P$2$\_nan) of missing values which are imputed with feature-wise average values.

Subsequently, a single data point $\mathbf{X}_{1j}$ of P$1$ is obtained as follows: P$1$ samples P$1$\_sample\_size data points~(e.g., housing records for CaliH/KingH) with replacement from a restricted subset~(of size P$1$\_data) of the original data, and uses these sampled P$1$\_sample\_size data points to compute a least-squares solution as P$1$'s one data point $\mathbf{X}_{1j}$ in the learning game. A single data point $\mathbf{X}_{2j}$ of P$2$ is obtained from observing a noisy sample from $\mathcal{N}(\hat \theta_2, \mathbf{1} \times \sigma^2)$ where $\hat \theta_2$ is the least-squares solution from P$2$'s data~(with missing values imputed) and $\sigma$ is known.

% As described in the experimental results section, P$1$ samples from a restricted subset of the original data, and each time samples a subset of size $\alpha$. P$2$'s data have missing observations, and this ratio is denoted by $\beta$, and are imputed with feature-wise average values.
%  In this experiment, P$1$'s data are obtained by sampling from a restricted subset of the original data. The size of this subset is denoted by P$1$\_data. This simulates the situation where P$1$ can only observe data from a certain area or district. When the sample size is high, its data will likely contain identical points. Two sampling methods are considered: the standard uniform sampling~(denoted as iid) and a method based on the statistical leverage scores~\cite{Drineas2012-leverageiid}~(denoted as lvg\_iid). P$2$'s data, on the other hand, are sampled uniformly from the whole of original data, but with some entries removed (values for some features are missing). The fraction of missing entries is denoted by P$2$\_nan. In other words, P$2$ can observe points from the whole space but its observations are defective.

\textbf{Additional results for 2-player.}
Figs.~\ref{fig:KingH} and~\ref{fig:MNIST_VAE} show the Shapley values and the cumulative counts for KingH and MNIST\_VAE, respectively. The statistics on $\delta$ are tabulated in Tables~\ref{table:CaliH-additional},~\ref{table:FaceA},~\ref{table:KingH}, and~\ref{table:MNIST_VAE} for the additional results on CaliH, FaceA, KingH, and MNIST\_VAE, respectively. Fig.~\ref{fig:KingH} shows a similar trend with previous results and verifies our result, while Fig.~\ref{fig:MNIST_VAE} displays more fluctuations within the overall converging trend. Note from the setting for Fig.~\ref{fig:MNIST_VAE} that though P$1$ has fewer data than P$2$, P$1$'s data are more balanced between the two digits.
The fluctuations in Fig.~\ref{fig:MNIST_VAE} suggest that experiment~$2$ is more difficult. This may be attributed to the fact that the parameter estimates for linear regression~(in experiment~$1$) are more analytically tractable than the mean estimates in a learned latent space implicitly represented by a VAE.

From Table~\ref{table:FaceA}, we observe that for FaceA, the default setting seems to `favor' P$2$: P$1$ needs more data to be able to contribute meaningfully~(left table) while increasing P$2$'s data further makes it hard for the Shapley values to converge~(right table).
On the other hand, from Table~\ref{table:KingH}, we observe that for KingH, the default setting for both P$1$ and P$2$ `favors' P$1$: Increasing P$1$'s data further makes it harder for the Shapley values to converge~(left table) and P$2$ needs more data to be able to meaningfully contribute~(right table).

From Table~\ref{table:MNIST_VAE}~(left), we observe that as the bias in both players' data is comparable~(first three rows), it takes fewer iterations for their Shapley values to converge. However, as P$1$'s data becomes more balanced~(last two rows), it becomes more difficult for P$2$ to collect sufficient data points to match, so it is more difficult for their Shapley values to converge. 
From Table~\ref{table:MNIST_VAE}~(right), we observe that if the quantities of both players' data differ significantly, it is in general difficult for their Shapley values to converge since all except one case~(second row) were not able to converge within $50$ iterations.

\begin{figure*}
    \begin{minipage}{0.47\textwidth}
        \centering
        \includegraphics[width=0.43\textwidth]{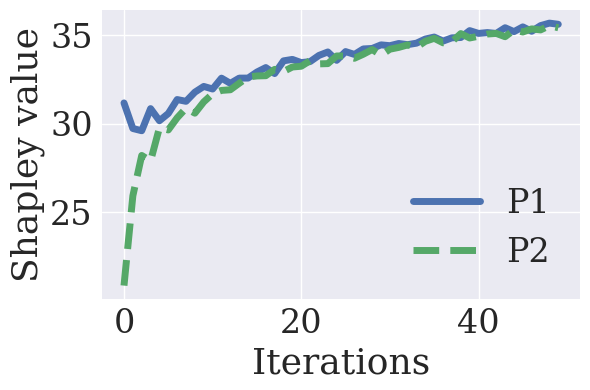}
        \hspace{0.05\textwidth}
        \includegraphics[width=0.43\textwidth]{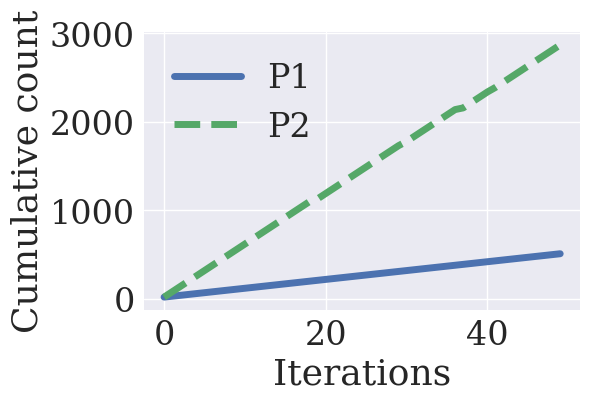}
    \caption{SV~(cumulative count) vs.~iterations on \textbf{KingH} where 
    P$1$\_data\ $=1000$, P$1$\_sample\_size\ $=100$, P$1$'s sampling is iid, and P$2$\_ratio\ $=$\ P$2$\_nan\ $=0.1$.
    }
    \label{fig:KingH}
    \end{minipage}
    \hfill
    \begin{minipage}{0.47\textwidth}
        \centering
        \includegraphics[width=0.43\textwidth]{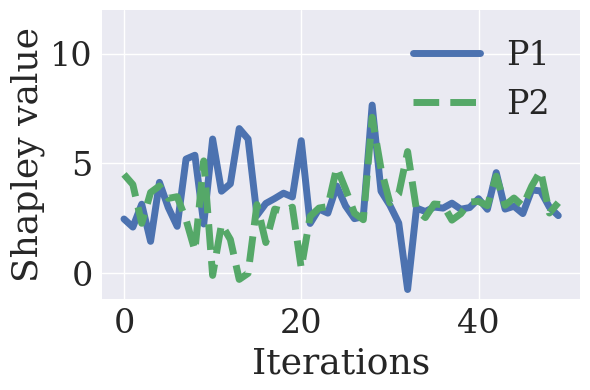}
        \hspace{0.05\textwidth}
        \includegraphics[width=0.43\textwidth]{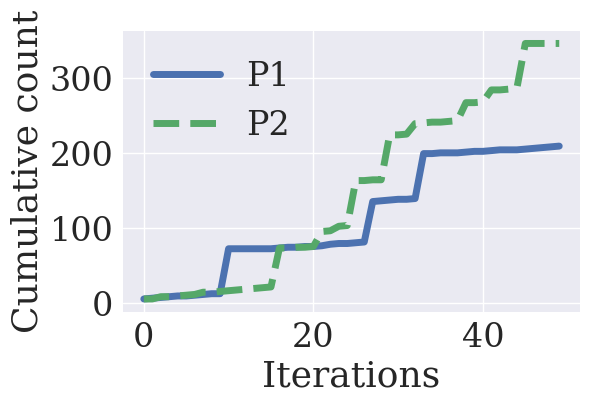}
    \caption{SV~(cumulative count) vs.~iterations on \textbf{MNIST\_VAE} where
    P$1$\_data\ $=1000$, P$2$\_data\ $=5000$, P$1$\_ratio\ $=0.2$, and P$2$\_ratio\ $=0.9$.
    }
    \label{fig:MNIST_VAE}
    \end{minipage}
\end{figure*}

\begin{table}
\centering
\caption{Statistics on $\delta$ for \textbf{CaliH} with varying P$2$'s contributions.
Iter denotes the smallest number of iterations for $\delta$ to be smaller than $0.1$ for $5$ consecutive iterations and * denotes cases where it is not satisfied. Also,
P$1$\_data\ $=5000$, P$1$\_sample\_size\ $=500$, and P$1$'s sampling is lvg\_iid.
}
\setlength{\tabcolsep}{4pt}
\resizebox{0.45\textwidth}{!}{
\begin{tabular}{rrrrrl}
\toprule
  P$2$\_nan & P$2$\_ratio & Lowest &  Average &  StDev & Iter \\
\midrule
    0.05 &   0.01 & 0.000 &    0.041 &  0.009 &            5 \\
    0.10 &   0.01 &    0.001 &    0.018 &  0.009 &            5 \\
    0.40 & 0.01 &      0.000 &    0.008 &  0.006 &            5 \\
    0.05 & 0.10 &      0.013 &    0.025 &  0.005 &            5 \\
    0.10 &   0.10 &    0.016 &    0.032 &  0.006 &            5 \\
    0.40 &   0.10 &    0.000 &    0.007 &  0.006 &            5 \\
    0.05 &   0.50 &    0.063 &    0.074 &  0.009 &            8 \\
    0.10 &   0.50 &    0.093 &    0.109 &  0.009 &          50* \\
    0.40 &   0.50 &    0.118 &    0.133 &  0.010 &          50* \\
\bottomrule
\end{tabular}
}
\label{table:CaliH-additional}
\end{table}

\begin{table}
\caption{Statistics on $\delta$ for \textbf{FaceA} with varying P$1$'s~(P$2$'s) contributions on the left~(right). Iter denotes the smallest number of iterations for $\delta$ to be smaller than $0.1$ for $5$ consecutive iterations and * denotes cases where it is not satisfied. 
For left table, P$2$\_ratio\ $=$\ P$2$\_nan\ $=0.1$.
For right table, P$1$\_data\ $=10000$, P$1$\_sample\_size\ $=100$, and P$1$'s sampling is iid.
}
\setlength{\tabcolsep}{4pt}
\resizebox{0.53\textwidth}{!}{
\begin{tabular}{rrlrrrl}
\toprule
    P$1$\_data &   P$1$\_sample\_size & sampling &  Lowest &  Average &  StDev &   Iter \\
\midrule
    1000 &        10 &         iid &   0.392 &    0.427 &  0.031 &          50* \\
    1000 &        10 &     lvg\_iid &   0.397 &    0.447 &  0.039 &          50* \\
    1000 &       100 &         iid &   0.006 &    0.025 &  0.006 &            5 \\
    1000 &       100 &     lvg\_iid &   0.000 &    0.005 &  0.004 &            5 \\
    1000 &       500 &         iid &   0.055 &    0.067 &  0.007 &            5 \\
    1000 &       500 &     lvg\_iid &   0.017 &    0.036 &  0.008 &            5 \\
    5000 &        10 &         iid &   0.377 &    0.429 &  0.033 &          50* \\
    5000 &       500 &         iid &   0.028 &    0.042 &  0.008 &            5 \\
    5000 &       500 &     lvg\_iid &   0.021 &    0.042 &  0.009 &            5 \\
\bottomrule
\end{tabular}}
\hfill
\resizebox{0.42\textwidth}{!}{
\begin{tabular}{rrrrrl}
\toprule
   P$2$\_nan & P$2$\_ratio & Lowest &  Average &  StDev & Iter \\
\midrule
    0.05 &   0.01 &      0.065 &    0.080 &  0.011 &           13 \\
    0.10 &   0.01 &      0.065 &    0.084 &  0.011 &           16 \\
    0.40 &   0.01 &       0.065 &    0.088 &  0.013 &           24 \\
    0.05 &   0.10 &      0.008 &    0.021 &  0.007 &            5 \\
    0.10 &   0.10 &      0.016 &    0.030 &  0.006 &            5 \\
    0.40 &   0.10 &       0.029 &    0.044 &  0.007 &            5 \\
    0.05 &   0.50 &       0.097 &    0.115 &  0.014 &          50* \\
    0.10 &   0.50 &       0.123 &    0.142 &  0.014 &          50* \\
    0.40 &   0.50 &      0.167 &    0.187 &  0.016 &          50* \\
\bottomrule
\end{tabular}
}
\label{table:FaceA}
\end{table}

\begin{table}
\setlength{\tabcolsep}{4pt}
\caption{Statistics on $\delta$ for \textbf{KingH} with varying P$1$'s~(P$2$'s) contributions on the left~(right). Iter denotes the smallest number of iterations for $\delta$ to be smaller than $0.1$ for $5$ consecutive iterations and * denotes cases where it is not satisfied.
For left table, P$2$\_ratio\ $=$\ P$2$\_nan\ $=0.1$.
For right table, P$1$\_data\ $=10000$, P$1$\_sample\_size\ $=100$, and P$1$'s sampling is iid.
}
\resizebox{0.53\textwidth}{!}{
\begin{tabular}{rrlrrrl}
\toprule
    P$1$\_data &   P$1$\_sample\_size & sampling &  Lowest &  Average &  StDev &   Iter \\
\midrule
    1000 &        10 &         iid &   0.000 &    0.015 &  0.011 &            5 \\
    1000 &        10 &     lvg\_iid &   0.000 &    0.015 &  0.014 &            5 \\
    1000 &       100 &         iid &   0.000 &    0.005 &  0.004 &            5 \\
    1000 &       100 &     lvg\_iid &   0.000 &    0.005 &  0.005 &            5 \\
    1000 &       500 &         iid &   0.033 &    0.173 &  0.024 &          50* \\
    1000 &       500 &     lvg\_iid &   0.125 &    0.140 &  0.012 &          50* \\
    5000 &        10 &         iid &   0.000 &    0.015 &  0.013 &            5 \\
    5000 &       500 &         iid &   0.133 &    0.159 &  0.011 &          50* \\
    5000 &       500 &     lvg\_iid &   0.098 &    0.158 &  0.015 &          50* \\
\bottomrule
\end{tabular}}
\hfill
\resizebox{0.42\textwidth}{!}{
\begin{tabular}{rrrrrl}
\toprule
   P$2$\_nan & P$2$\_ratio & Lowest &  Average &  StDev & Iter \\
\midrule
    0.05 & 0.01 &      0.192 &    0.217 &  0.017 &          50* \\
    0.10 & 0.01 &      0.189 &    0.207 &  0.014 &          50* \\
    0.40 & 0.01 &      0.177 &    0.193 &  0.013 &          50* \\
    0.05 & 0.10 &       0.000 &    0.005 &  0.004 &            5 \\
    0.10 & 0.10 &       0.000 &    0.006 &  0.006 &            5 \\
    0.40 & 0.10 &      0.000 &    0.005 &  0.005 &            5 \\
    0.05 & 0.50 &       0.000 &    0.006 &  0.004 &            5 \\
    0.10 & 0.50 &       0.000 &    0.009 &  0.006 &            5 \\
    0.40 & 0.50 &      0.000 &    0.010 &  0.006 &            5 \\
\bottomrule
\end{tabular}
}
\label{table:KingH}
\end{table}

\begin{table*}
\setlength{\tabcolsep}{4pt}
\caption{Statistics on $\delta$ for \textbf{MNIST\_VAE} with varying P$1$'s~(P$2$'s) contributions on the left~(right). Iter denotes the smallest number of iterations for $\delta$ to be smaller than $0.1$ for $5$ consecutive iterations and * denotes cases where it is not satisfied.
}
\resizebox{0.50\textwidth}{!}{
\begin{tabular}{cclrrrl}
\toprule
 P$1$/P$2$\_data &  P$1$/P$2$\_ratio &  Lowest &  Average &  StDev & Iter \\
\midrule
    1000 / 1000 &         0.1 / 0.9 &   0.000 &    0.356 &  0.388 &           28 \\
    1000 / 1000 &         0.2 / 0.9 &   0.002 &    0.287 &  0.376 &           19 \\
    1000 / 1000 &         0.3 / 0.9 &   0.004 &    0.245 &  0.337 &           26 \\
    1000 / 1000 &         0.4 / 0.9 &   0.035 &    0.378 &  0.320 &          50* \\
    1000 / 1000 &         0.5 / 0.9 &   0.004 &    0.418 &  0.447 &          50* \\
\bottomrule
\end{tabular}}
\hfill
\resizebox{0.50\textwidth}{!}{
\begin{tabular}{ccrrrl}
\toprule
 P$1$/P$2$\_data &  P$1$/P$2$\_ratio &  Lowest &  Average &  StDev & Iter \\
\midrule
    1000 / 5000 &         0.1 / 0.9 &   0.006 &    0.202 &  0.204 &          50* \\
    1000 / 5000 &         0.2 / 0.9 &   0.007 &    0.226 &  0.332 &           38 \\
    1000 / 5000 &         0.3 / 0.9 &   0.000 &    0.345 &  0.366 &          50* \\
    1000 / 5000 &         0.4 / 0.9 &   0.008 &    0.365 &  0.334 &          50* \\
    1000 / 5000 &         0.5 / 0.9 &   0.006 &    0.329 &  0.334 &          50* \\
\bottomrule
\end{tabular}
}
\label{table:MNIST_VAE}
\end{table*}

\textbf{Additional results for multi-player.}
Figs.~\ref{fig:CaliH-multi},~\ref{fig:FaceA-multi}, and~\ref{fig:MNIST_VAE-multi} plot the Shapley values and cumulative counts for CaliH, FaceA, and MNIST\_VAE in multi-player scenarios, respectively.
It can be observed from Fig.~\ref{fig:MNIST_VAE-multi} that the convergence of the Shapley values is not very obvious, which corresponds to our previous observations for the $2$-player scenario in Fig.~\ref{fig:MNIST_VAE} where there are more fluctuations and suggests that this is a challenging setting.

\begin{figure*}
    \begin{minipage}{0.47\textwidth}
        \centering
        \includegraphics[width=0.43\textwidth]{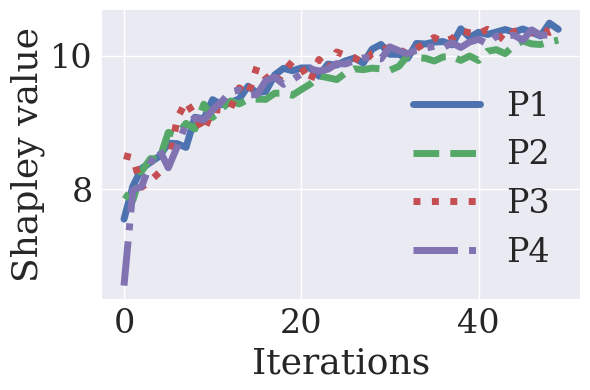}
        \hspace{0.05\textwidth}
        \includegraphics[width=0.43\textwidth]{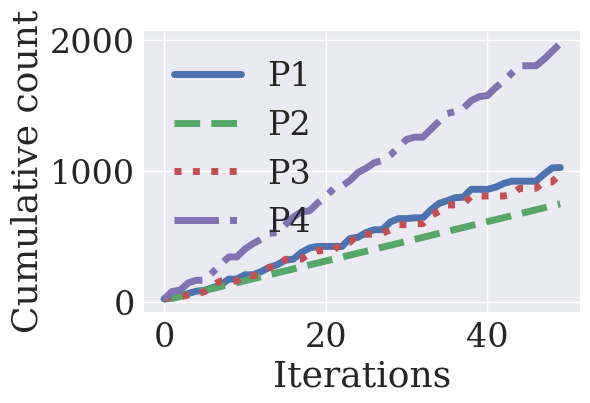}
    \caption{SV~(cumulative count) vs.~iterations on \textbf{CaliH} for a multi-player scenario where
    P$1$/P$3$\_data\ $=1000$, P$1$/P$3$\_sampling\_size\ $=500$, P$2$/P$4$\_ratio\ $=0.05$, and P$2$/P$4$\_nan\ $=0.3$.
    }
    \label{fig:CaliH-multi}
    \end{minipage}
    \hfill
    \begin{minipage}{0.47\textwidth}
        \centering
        \includegraphics[width=0.43\textwidth]{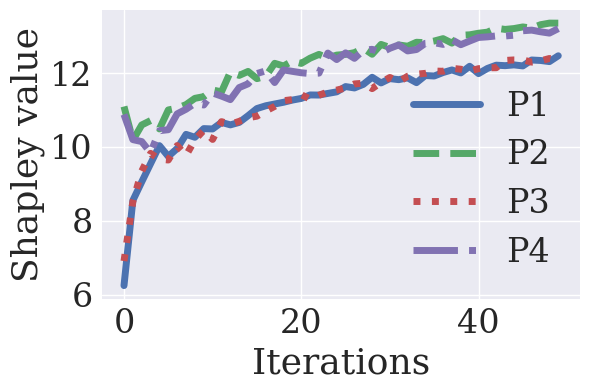}
        \hspace{0.05\textwidth}
        \includegraphics[width=0.43\textwidth]{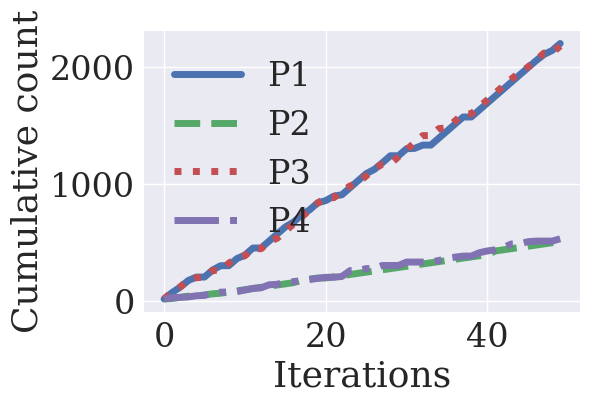}
    \caption{SV~(cumulative count) vs.~iterations on \textbf{FaceA} for a multi-player scenario where 
    P$1$/P$3$\_data\ $=5000$, P$1$/P$3$\_sampling\_size\ $=100$, P$2$/P$4$\_ratio\ $=0.1$, and P$2$/P$4$\_nan\ $=0.2$.
    }
    \label{fig:FaceA-multi}
    \end{minipage}
    \begin{minipage}{0.47\textwidth}
        \centering
        \includegraphics[width=0.43\textwidth]{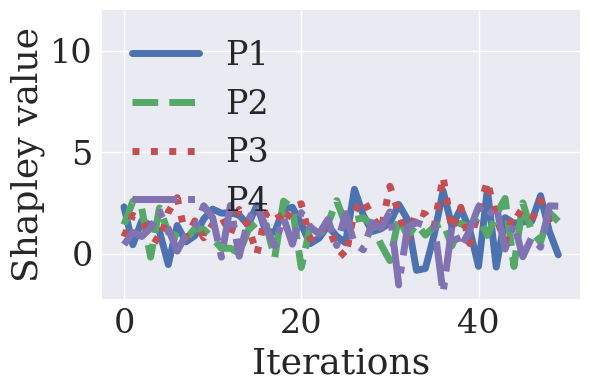}
        \hspace{0.05\textwidth}
        \includegraphics[width=0.43\textwidth]{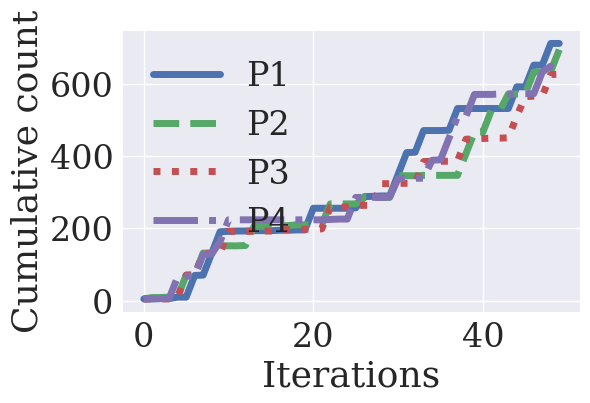}
    \caption{SV~(cumulative count) vs.~iterations on \textbf{MNIST\_VAE} for a multi-player scenario where 
    P$1$/P$3$\_data\ $=1000$, P$1$\_ratio\ $=0.1$, P$2$/P$4$\_data\ $=5000$, and P$2$\_ratio\ $=0.1$.
    }
    \label{fig:MNIST_VAE-multi}
    \end{minipage}
\end{figure*}

\end{document}